\documentclass{article} 
\usepackage{iclr2027_conference,times}

\usepackage{amsmath,amsfonts,bm}

\def\eqref#1{equation~\ref{#1}}

\def\1{\bm{1}}

\def\vs{{\bm{s}}}

\DeclareMathAlphabet{\mathsfit}{\encodingdefault}{\sfdefault}{m}{sl}
\SetMathAlphabet{\mathsfit}{bold}{\encodingdefault}{\sfdefault}{bx}{n}

\usepackage[hyphens]{url}
\usepackage{graphicx}
\usepackage{hyperref}
\definecolor{linkblue}{HTML}{1A6FC4}
\hypersetup{colorlinks=true,urlcolor=linkblue,linkcolor=black,citecolor=black}
\usepackage{caption}
\usepackage{subcaption}
\usepackage{amsmath,amssymb,bm}
\usepackage{enumitem}
\usepackage{algorithm,algorithmic}
\usepackage{multirow,makecell}
\usepackage{booktabs}
\usepackage{tabularx}
\usepackage[table,dvipsnames]{xcolor}
\usepackage{colortbl}
\usepackage{longtable}
\usepackage{placeins}
\usepackage{amsthm}
\usepackage{microtype}

\usepackage{float}

\definecolor{janusblue}{HTML}{DCEBFA}
\definecolor{janusgreen}{HTML}{E9F6EE}
\definecolor{janusamber}{HTML}{F5C244}
\definecolor{janussecond}{HTML}{7E93AD}
\definecolor{janusrule}{HTML}{9AA9B7}
\definecolor{janusbest}{HTML}{16518F}
\definecolor{janussecondbest}{HTML}{67AACF}
\newcommand{\appendixtablesetup}{\scriptsize\renewcommand{\arraystretch}{0.96}\setlength{\tabcolsep}{2.4pt}}
\newcommand{\wideappendixtablesetup}{\scriptsize\renewcommand{\arraystretch}{0.98}\setlength{\tabcolsep}{2.2pt}}

\newcommand{\maintablesetup}{\footnotesize\renewcommand{\arraystretch}{1.13}\setlength{\tabcolsep}{4.5pt}\arrayrulecolor{janusrule}}
\newcommand{\tightmaintablesetup}{\footnotesize\renewcommand{\arraystretch}{1.08}\setlength{\tabcolsep}{2.8pt}\arrayrulecolor{janusrule}}
\newcommand{\bestcell}[1]{\cellcolor{janusblue!70}\textbf{#1}}
\newcommand{\secondcell}[1]{\cellcolor{janusgreen!70}#1}
\newcommand{\cmark}{\checkmark}

\setlist{itemsep=0pt,topsep=2pt,parsep=0pt}
\newtheorem{proposition}{Proposition}

\title{JANUS: Online Jacobian-Aligned Infill\\for Black-Box Optimization}

\author{%
Hongyuan Yu\textsuperscript{1}, Pufan Xu\textsuperscript{2}\thanks{Corresponding author.}, Jiaojiao Yi\textsuperscript{1}, Yiding Tian\textsuperscript{1}, Mingrui Sun\textsuperscript{1}, Jiayuan Lu\textsuperscript{1} \\[2pt]
\textbf{Changyuan Wen\textsuperscript{1}} \\[2pt]
\textsuperscript{1}Multimedia Department, Xiaomi Inc. \quad \textsuperscript{2}School of Integrated Circuits, Tsinghua University \\[2pt]
{\small\texttt{\{yuhongyuan, yijiaojiao, tianyiding,}} \\
{\small\texttt{sunmingrui1, lujiayuan, wenchangyuan\}@xiaomi.com}} \\
{\small\texttt{xpf22@mails.tsinghua.edu.cn}}%
}

\iclrfinalcopy 
\begin{document}
\raggedbottom
\renewcommand{\topfraction}{0.92}
\renewcommand{\bottomfraction}{0.78}
\renewcommand{\textfraction}{0.06}
\renewcommand{\floatpagefraction}{0.85}
\setlength{\textfloatsep}{5pt plus 1pt minus 2pt}
\setlength{\floatsep}{5pt plus 1pt minus 2pt}
\setlength{\intextsep}{5pt plus 1pt minus 2pt}
\setlength{\abovecaptionskip}{2pt}
\setlength{\belowcaptionskip}{1pt}
\captionsetup{skip=2pt}
\maketitle
\lhead{}\rhead{}\chead{} 

\begin{abstract}
Population optimizers such as CMA-ES, DE, and multi-objective evolutionary algorithms drive search mainly through selection signals that are scalar or rank based: such a signal indicates that one candidate outperforms another, but not the local direction responsible for the improvement. JANUS (\emph{Jacobian-Aligned Newton-Unified Search}) is a plug-and-play infill module that extracts this missing local geometric signal without replacing the host optimizer. It estimates a local Jacobian from the recent evaluation trace; the same Jacobian yields both a damped Gauss--Newton exploitation candidate and a trace-preserving exploration metric, reserving a fraction of the host's per-generation candidate slots for geometry-guided infill rather than spending evaluations on top of the host's budget. Unlike MetaBBO methods, JANUS needs no offline training or task distribution, estimating this geometry on the fly from the current run alone, while the host keeps full control of selection, survival, covariance adaptation, and step-size control. Under same-protocol comparisons, JANUS improves the CMA-ES host on \textbf{11--15/16} BBOB functions across $d\in\{30,100,500\}$. It also attains the best mean error on \textbf{13 of the 16} functions at $d{=}500$ in the complete NN-BBO/MetaBBO baseline comparison, with no training cost, and yields a $936\times$ geometric-mean improvement over the host on a $d{=}1000$ BBOB subset. On structured and multi-objective tasks, JANUS gives the best mean cost on 1135-dimensional UAV path planning ($-12.8\%$ vs.\ the strongest baseline), and it improves SMS-EMOA/AGE-MOEA2 hosts on 12/38 multi-objective tasks with zero significant regressions. Code is available at \url{https://github.com/hongyuanyu/JANUS}.
\end{abstract}

\section{Introduction}
\label{sec:intro}

Black-box optimization (BBO) appears in hyperparameter tuning, simulation calibration, controller design, and scientific workflows where gradients are unavailable or unreliable~\citep{JMLR:v13:bergstra12a,hansen2016cma,ma2023metabox}. Population optimizers are successful partly because they avoid brittle modeling assumptions: they rank candidates by scalar fitness and adapt from selection statistics. Yet scalar rankings are a deliberately thin interface. A ranking indicates that one sample is better, but it does not reveal the local direction responsible for the improvement, nor how the next generation of samples should be shaped. We test this gap across scalar BBOB, structured UAV planning, and multi-objective optimization (Fig.~\ref{fig:dashboard}).

\begin{figure}[t]
\centering
\includegraphics[width=0.8\linewidth]{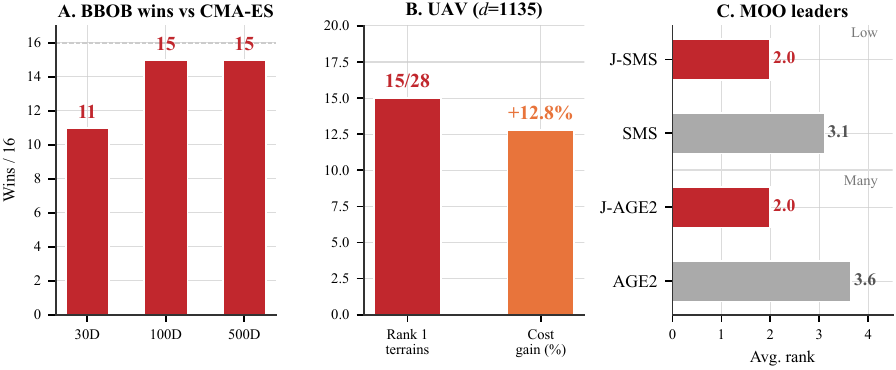}
\caption{\textbf{JANUS at a glance.} Headline BBOB, UAV, and MOO results.}
\label{fig:dashboard}
\end{figure}

\paragraph{Population hosts leave local geometry unused.}
We identify a \emph{rank-geometry gap} between two families of derivative-free methods. On one side, population hosts (e.g., CMA-ES, DE) adapt search distributions using selection signals that are scalar or rank based. They are sample-efficient and robust, but they discard the local residual geometry that classical trust-region methods exploit. On the other side, derivative-free trust-region methods (NEWUOA, BOBYQA, DFO-LS) build local quadratic or least-squares models from the trace. These methods replace the host optimizer entirely, however, and spend a large fraction of the budget on model maintenance. Restricted-covariance variants such as LM-CMA~\citep{loshchilov2014lm} and VkD-CMA~\citep{akimoto2014vkd} partially address scalability by learning a restricted covariance basis, but they still learn this basis from host selection statistics rather than from the local residual Jacobian itself. We show (Prop.~\ref{prop:restricted}) that a covariance basis fixed independently of the local residual Jacobian $J$ cannot, in general, align its exploration metric with the inverse Hessian $T=(J^\top J+\lambda I)^{-1}$ of the local residual model.

\paragraph{Closing the gap with one Jacobian, two uses.}
JANUS turns the recent search trace into residual coordinates and estimates a local Jacobian $J$. If evaluations expose components, for example
\begin{equation}
    f(x) = \tfrac{1}{m}\textstyle\sum_{i=1}^{m} r_i(x),
    \label{eq:component}
\end{equation}
we use normalized residual coordinates such as $s_i(x)=\sqrt{r_i(x)}$; if only scalar values are observed, JANUS instead derives residual coordinates from local linear or quadratic fits on the recent window. Explicit components are useful but not required: what matters is a local residual map $s(x)$ whose Jacobian summarizes how recent changes in $x$ affect the local objective model.

The central observation is that the same $J$ answers two questions usually handled separately: no prior method derives both a bounded exploitation candidate and a trace-preserving exploration metric from the same online computation. Earlier per-component $\arg\min$ assembly can be viewed as a block-diagonal special case restricted to exploitation: it assembles the best fragment from each component, but discards the coupling across coordinates and says nothing about population spread. JANUS keeps the full $J^\top J$ and reuses it for both exploitation and exploration.

JANUS is implemented as \emph{infill}, not a replacement optimizer: it asks the host for native candidates, adds a few bounded Jacobian-aligned candidates, evaluates the union, and lets the host perform its usual selection, survival, covariance, and step-size updates. Because JANUS only contributes candidates, it can augment many population-based search algorithms without meta-training or task-specific tuning, while preserving the host optimizer's safeguards.

\paragraph{On-instance, not meta-learning.} JANUS estimates the residual Jacobian online from the current black-box problem and discards it after the run: there is no training distribution, no learned policy, and no parameter transfer between problems. This trade-off is deliberate. MetaBBO methods~\citep{ma2023metabox,lange2023discovering,ICLR2024_b742b708,Li_Wu_Zhang_Wang_2025} can outperform on-instance methods when the test task matches the training distribution, but they introduce a generalization axis and require offline meta-training; JANUS occupies the complementary regime, working on any single task without prior training. Whether JANUS's local-geometry signal could be augmented with a learned prior over residual structures is left for future work. We validate this across three problem classes, four host optimizers, and one industrial deployment. Our contributions are:
\begin{enumerate}[leftmargin=1.6em,itemsep=1pt]
\item We show (Prop.~\ref{prop:restricted}) that restricted-covariance variants (LM-CMA, VkD-CMA) cannot, in general, align a host-fixed covariance basis with the local curvature induced by a residual Jacobian.
\item We introduce online construction of residual coordinates to extract reusable local geometry from black-box search traces, using either observed components or coordinates from a local linear or quadratic fit; from one estimated Jacobian, we derive two bounded infill mechanisms: a damped Gauss--Newton exploitation candidate and a trace-preserving inverse Hessian exploration distribution.
\item Under same-protocol comparisons, JANUS variants improve over the CMA-ES host on 11--15/16 BBOB functions across $d\in\{30,100,500\}$ (Table~\ref{tab:bbob_dense_30_100_std}, Table~\ref{tab:bbob_dense_500d}) and achieve the best mean cost on 28-terrain UAV path planning ($-12.8\%$ vs the strongest baseline). The same mechanism improves both SMS-EMOA and AGE-MOEA2 hosts on 12/38 multi-objective tasks with zero significant regressions, and transfers to a DE host (86--96\% win rate).
\end{enumerate}

\section{Related Work}
\label{sec:related}

\paragraph{Evolution strategies and covariance adaptation.}
CMA-ES adapts a Gaussian search distribution from selected samples and remains one of the strongest general-purpose black-box optimizers~\citep{hansen2016cma}. Scalable variants such as LM-CMA~\citep{loshchilov2014lm} and VkD-CMA~\citep{akimoto2014vkd} restrict the covariance basis to reduce cost in high dimension, but still learn geometry from selection statistics. FOCAL~\citep{10.1145/1830483.1830563} modifies the adaptation rule itself to recover inverse-Hessian structure near the optimum. JANUS is complementary: it estimates residual geometry externally and injects candidates without changing the host update rule.

\paragraph{Model-based derivative-free optimization.}
Classical DFO methods build local interpolation, trust-region, or least-squares models from black-box evaluations, including NEWUOA, BOBYQA, COBYLA, and DFO-LS~\citep{powell2006newuoa,powell2009bobyqa,cartis2019dfols}. They are most effective when model maintenance is affordable and the local model can serve as the entire optimizer. JANUS does not maintain a trust region, choose incumbents, or replace the search distribution; it imports only the residual-model signal into a population loop, so the host retains global search pressure while JANUS spends a small infill budget on geometry-aware proposals.

\paragraph{Natural-gradient and learned optimizer views.}
Natural evolution strategies and information-geometric optimization interpret distribution updates as gradient steps on a statistical manifold~\citep{wierstra2014natural,ollivier2017information}; CMA-ES itself can be viewed through approximate natural-gradient adaptation~\citep{akimoto2010bidirectional}. JANUS shares the idea that proposals should encode useful geometry, but its geometry source is a residual normal matrix $J^\top J$ rather than an expected-fitness gradient estimated from rankings. MetaBBO methods learn optimizers or operators across task distributions~\citep{ma2023metabox,lange2023discovering,ICLR2024_b742b708,Li_Wu_Zhang_Wang_2025}; JANUS instead estimates geometry on the current instance only, avoiding offline meta-training and task-distribution shift.

\paragraph{Multi-objective optimization and infill.}
Multi-objective evolutionary algorithms maintain convergence and diversity through nondominated sorting, hypervolume selection, decomposition, reference vectors, or geometry-aware survival~\citep{deb2002nsga,beume2007sms,zhang2007moead,deb2014nsga3,cheng2016rvea}; expensive MOO also studies infill criteria such as expected hypervolume improvement~\citep{emmerich2006single}. JANUS does not replace survival or reference-vector logic. Attached to SMS-EMOA or AGE-MOEA2~\citep{panichella2022agemoea2}-style hosts, it only changes how a small fraction of offspring are proposed, using the objective vector itself as residual coordinates. Surrogate-assisted evolutionary algorithms~\citep{jin2011surrogate} instead build global surrogates (GP, RBF) to pre-screen candidates; JANUS uses only a local fit on the recent window, scaling to $d{>}1000$.

\section{Method}
\label{sec:method}

Table~\ref{tab:main_notation} gives the symbols needed to read this section; App.~\ref{app:notation} keeps the full notation and variant list.

\subsection{Host interface}
JANUS assumes only a population-style ask/evaluate/tell loop. At generation $g$, a host such as CMA-ES samples candidates (e.g., $x_k\sim\bar x+\sigma\mathcal{N}(0,C)$), evaluates them, and updates its internal state from selected or surviving samples. JANUS does not change this loop: it reads a recent evaluation window, proposes a small number of additional candidates, and returns all evaluated points to the host, so selection pressure, survival logic, covariance adaptation, and step-size control remain host-defined. A host baseline and its JANUS-augmented counterpart always share the same total evaluation budget; JANUS reallocates part of each generation's candidate slots to infill proposals, rather than drawing evaluations on top of the budget.

\subsection{Online residual coordinates}

The only object JANUS needs is a local residual map $s(z)\in\mathbb{R}^m$ on a recent window $\mathcal{W}=\{(x_i,f_i)\}_{i=1}^n$. We first express candidates in host-normalized coordinates, $z_i=D^{-1}(x_i-\mu)$, using the current local scale. JANUS then constructs $s$ in one of three ways:
\begin{itemize}[leftmargin=1.4em,itemsep=1pt]
\item \textbf{Observed components}: if evaluations expose nonnegative per-scenario, per-constraint, or per-objective terms, use normalized residual coordinates such as $s_i(x)=\sqrt{r_i(x)}$; otherwise use the shifted or standardized component value.
\item \textbf{Linear scalar mode}: if only scalar values are observed, fit $\hat f(z)=a^\top z+b$ on $\mathcal{W}$ and use $s_1(z)=a^\top z$ after normalization. This gives a stable slope estimate for the GN candidate.
\item \textbf{Quadratic scalar mode}: fit $\hat f(z)=z^\top A z+b^\top z+c$, decompose $A=\sum_i\lambda_i u_i u_i^\top$, and use $s_i=\sqrt{|\lambda_i|}\,u_i^\top z$ as curvature-aligned coordinates.
\end{itemize}
The scalar modes are not fallbacks; they handle the common case where the benchmark exposes only scalar values. In these modes $s$ is a locally fitted \emph{pseudo-residual} coordinate system, not a Gauss--Newton decomposition of the objective: the linear mode drops the intercept $b$ (encoding only the local slope direction), and the quadratic mode uses $\sqrt{|\lambda_i|}$, discarding the curvature sign, so $\|s\|^2$ does not reproduce $\hat f$. It is a heuristic slope- and curvature-aligned frame for generating bounded candidates, not an exact residual factorization of $f$. Linear mode is conservative and mainly supports exploitation; quadratic mode supplies multiple curvature directions for CMIX but can be noisy on strongly multimodal landscapes.

Given residual-coordinate values on $\mathcal{W}$, JANUS fits $s(x_i)\approx b+Jz_i$ by ridge regression and forms $G=J^\top J$, the Gauss--Newton approximation to the Hessian of the local residual model $\|s(x)\|^2$. The regression is intentionally local: old samples are discarded, inputs are host-normalized, and ridge regularization suppresses unsupported directions. Unlike global surrogate optimization, JANUS does not require the model to predict the objective across the whole box — only a local coordinate system that produces a bounded candidate and a temporary proposal metric. Survival remains the host's decision.

\begin{table}[t]
\centering
\footnotesize\renewcommand{\arraystretch}{1.02}\setlength{\tabcolsep}{3pt}
\caption{Core notation used in the method.}
\label{tab:main_notation}
\begin{tabular*}{\linewidth}{@{\extracolsep{\fill}} lp{0.62\linewidth}}
\toprule
Symbol & Meaning \\
\midrule
$x, f(x)$ & Candidate and evaluated black-box value. \\
$\mathcal{W}$ & Recent evaluated window used for local fitting. \\
$r(x)$ & Observed components/objective vector when available. \\
$s(x)$ & Residual coordinates, observed or scalar-derived. \\
$J$ & Online residual Jacobian fitted on $\mathcal{W}$. \\
$G=J^\top J$ & Gauss--Newton Hessian approximation; metric for GN exploitation and CMIX exploration. \\
$\lambda,\rho$ & Damping coefficient and CMIX candidate fraction. \\
GN / CMIX & Bounded Newton candidate / inverse Hessian mixing. \\
JANUS-GN / JANUS & Exploitation-only variant / GN plus CMIX variant. \\
\bottomrule
\end{tabular*}
\end{table}

\subsection{Exploitation: damped Gauss--Newton infill}
Let $x_b$ be the best point in the current window and $s_b=s(x_b)$. JANUS forms the damped Gauss--Newton step
\begin{equation}
\Delta_{\mathrm{GN}}=-(G+\lambda I)^{-1}J^\top s_b,
\label{eq:gn_step}
\end{equation}
and proposes $x_b+\eta\Delta_{\mathrm{GN}}$. To avoid destabilizing CMA-style step-size control, the proposal is bounded-injected following the clipping rule of \citet{hansen2011injecting}. In whitened coordinates, its step is clipped to $\|y_e\|\le\kappa\sqrt{d}$. Thus the injected candidate has the scale of a normal host sample, while its direction comes from the residual Jacobian rather than from historical-best reuse or random perturbation.

\subsection{Exploration: inverse Hessian CMIX}
For exploration, JANUS inverts the same Gauss--Newton curvature $G$ to form a regularized inverse-Hessian covariance,
\begin{equation}
T=(G+\lambda_{\mathrm{mix}}I)^{-1},\qquad
C\leftarrow(1-w)C+w\frac{\operatorname{tr}(C)}{\operatorname{tr}(T)}T.
\label{eq:cmix}
\end{equation}
The trace factor preserves the total sampling energy of $C$ and changes only its shape: directions that are locally flat under $G$ receive more variance, while sharp directions receive less. Setting $w=0$ recovers JANUS-GN, the conservative exploitation-only fallback used when the curvature estimate is unstable (Sec.~\ref{sec:exp-rq4}).

\subsection{Dimension-adaptive mixing weight}
\label{sec:weight}
\begin{equation}
w=w_{\mathrm{base}}\cdot\frac{\min(m,|\mathcal{W}|)}{d},\qquad w\in[w_{\min},w_{\max}].
\label{eq:weight}
\end{equation}
The numerator measures how much independent residual information is available, while the denominator reflects the ambient dimension over which that information must be spread. The bounds $[w_{\min},w_{\max}]$ prevent the metric from either disappearing or dominating the host covariance. This rule is fixed across all experiments.

\subsection{Alternative Designs and Scope}
Standalone damped GN or DFO-LS spends the whole budget maintaining a local model, whereas JANUS contributes only bounded candidates and leaves selection and scale control to the host; the empirical gap is large, with standalone local models much weaker than host+JANUS under the same high-dimensional BBOB budgets (App.~\ref{app:baselines}). Restricted-covariance hosts such as LM-CMA and VkD-CMA are complementary rather than substitutes: their covariance basis is learned from host selection statistics, while JANUS derives directions from the residual Jacobian (Prop.~\ref{prop:restricted}; App.~\ref{app:restricted-cov}, Table~\ref{tab:restricted_cov_baselines}). Learning the mixing weight would instead add an offline task distribution and a new generalization axis; the fixed rule in Eq.~\ref{eq:weight}, by contrast, is shared across BBOB, UAV, and MOO.

\begin{algorithm}[t]
\caption{JANUS: online metric-guided infill}
\label{alg:janus}
\begin{algorithmic}[1]
\REQUIRE Host $\mathcal{A}$, window $n$, infill ratio $\rho$, damping $\lambda$
\WHILE{budget remains}
    \STATE Ask $\mathcal{A}$ for native candidates; reserve $\rho$ slots for infill
    \STATE Fit $J$ from window $\mathcal{W}$; form $G=J^\top J$
    \STATE GN candidate: $-(G+\lambda I)^{-1}J^\top s_b$; bounded inject
    \STATE CMIX candidates: sample from trace-normalized $(G+\lambda I)^{-1}$
    \STATE Evaluate all candidates; host performs selection/survival
\ENDWHILE
\end{algorithmic}
\end{algorithm}

\section{Theoretical Properties}
\label{sec:theory}

\begin{proposition}[A fixed rank-$k$ basis cannot express arbitrary residual-aligned metrics]
\label{prop:restricted}
Let $C_{\rm host}=B M B^\top$ be a purely rank-$k$ host covariance whose basis $B\in\mathbb{R}^{d\times k}$, $k<d$, has orthonormal columns and is fixed by host adaptation statistics independently of the local residual Jacobian $J\in\mathbb{R}^{m\times d}$. Let $T=(J^\top J+\lambda I)^{-1}$ be the inverse Hessian metric induced by $J$ ($\lambda>0$), and let $P_\perp=I-BB^\top$ be the orthogonal projector onto $\operatorname{span}(B)^\perp$. Then no restricted-covariance update of the form $C'=(1-\alpha)C_{\rm host}+\alpha B M' B^\top$ can equal $T$; quantitatively,
\[
\|C'-T\|_F\ \ge\ \|P_\perp T P_\perp\|_F\ \ge\ \frac{\sqrt{d-k}}{\sigma_{\max}(J)^2+\lambda}\ >\ 0,
\]
where $\sigma_{\max}(J)$ is the largest singular value of $J$. The bound is independent of $M,M',\alpha$ and of the host block, so no choice of restricted-covariance parameters removes it. The gap $\|P_\perp T P_\perp\|_F$ is exactly the residual-aligned curvature that lives in $\operatorname{span}(B)^\perp$, which the restricted family sets to zero regardless of its parameters.
\end{proposition}

\begin{proposition}[Damped GN descends on the fitted residual model]
\label{prop:descent-main}
Let $\hat s(x_b+\Delta)=s_b+J\Delta$ and $\hat f(\Delta)=\|s_b+J\Delta\|_2^2$. For $\lambda>0$, $\Delta_{\rm GN}=-(J^\top J+\lambda I)^{-1}J^\top s_b$ satisfies
\[
\nabla\hat f(0)^\top\Delta_{\rm GN}=-2s_b^\top J(J^\top J+\lambda I)^{-1}J^\top s_b\le0,
\]
strictly whenever $J^\top s_b\ne0$.
\end{proposition}

\begin{proposition}[Bounded injection and trace-preserving CMIX]
\label{prop:safety-main}
If an injected point is clipped in host-whitened coordinates to $\|y_e\|\le\kappa\sqrt d$, then $\|x_e-\bar x\|\le\kappa\sigma\lambda_{\max}(C^{1/2})\sqrt d$. Moreover, for $T\succ0$ and $C'=(1{-}w)C+w\operatorname{tr}(C)T/\operatorname{tr}(T)$, $\operatorname{tr}(C')=\operatorname{tr}(C)$.
\end{proposition}

These statements explain why JANUS is infill rather than a replacement optimizer. Prop.~\ref{prop:restricted}'s scope is restricted-basis hosts ($k<d$), true of LM-CMA/VkD-CMA by construction; for a full-rank host such as vanilla CMA-ES, $\operatorname{span}(B)=\mathbb{R}^d$ and the proposition makes no structural claim, so JANUS's advantage there is empirical rather than structural — learning a general full covariance from rank-based updates involves $O(d^2)$ covariance parameters and can therefore be sample- and computation-intensive in high dimensions, while JANUS forms the residual-aligned metric from the fitted $J$ in one step (App.~\ref{app:proofs} separates the structural regime, illustrated by Table~\ref{tab:restricted_cov_baselines}, from this convergence-speed regime). Prop.~\ref{prop:descent-main} is local and model-relative: the injected point descends on the fitted residual chart, but the chart may not describe disconnected basins. The failure mode is bounded rather than open-ended: under the Prop.~\ref{prop:safety-main} clip, the proposal displacement of a misspecified chart stays on the host sampling scale, although its objective value is not guaranteed to improve. Prop.~\ref{prop:safety-main} gives this safety interface: injected points stay on the host sampling scale and CMIX reshapes direction without changing covariance trace. A blockwise assembly rule is recovered as a degenerate zero-coupling GN case. Full proofs are in App.~\ref{app:proofs}.

\textbf{Local versus global exploration.} The inverse Hessian metric can reallocate variance among directions represented in the recent trace, but cannot identify disconnected basins outside the current sampling region; such exploration still requires host diversity, restarts, or another global mechanism. This predicts why CMIX is strongest when the trace contains coherent local curvature, while JANUS-GN is safer when basin discovery or noisy pseudo-residuals dominate.

\section{Experiments}
\label{sec:experiments}

All experiments report final objective values or hypervolume under fixed budgets and independent seeds. Appendix~\ref{app:protocols} gives the full search spaces, evaluation budgets, run counts, and metric conventions for BBOB, UAV, and MOO.

\subsection{Scalar BBOB: Geometry from Scalar Evaluations ($d\in\{30,100,500\}$)}
\label{sec:exp-bbob}

BBOB represents the most challenging setting for our claim: the benchmark exposes only scalar values, so any useful geometry must be recovered from the recent trace rather than observed directly. We use a MetaBox/COCO-compatible BBOB protocol~\citep{ma2023metabox}: search domain $[-5,5]^d$, budget 20,000 function evaluations, 30 independent runs, on the 16 held-out BBOB functions. The CMA-ES host comparison in Fig.~\ref{fig:bbob_heatmap} is the controlled headline: same functions, seeds, population size, initialization, box constraints, and budget, with and without JANUS infill. JANUS-GN injects only the Gauss--Newton candidate, while JANUS additionally enables CMIX.

\paragraph{Host-controlled comparison (headline).}
Fig.~\ref{fig:bbob_heatmap} compares JANUS variants against their CMA-ES host under the identical $[-5,5]^d$ protocol. With the BBOB-specific fuse cadence (fuse\_every=50, tuned once for $d\in\{30,100\}$; App.~\ref{app:hparams} shows all hyperparameters stay within $1.5\times$ of default under a one-at-a-time sweep), JANUS wins on 11/16 functions at $d{=}30$ (11W/1T/4L), \textbf{15/16} at $d{=}100$ (15W/1T/0L), and \textbf{15/16} at $d{=}500$ (geometric-mean ratio $0.11$), with 0 significant regressions at $d\in\{100,500\}$ (Table~\ref{tab:bbob_dense_30_100_std}). The few host-controlled losses concentrate on functions such as $f_{22}$ and $f_{23}$, where progress depends on global basin discovery beyond the current residual model. Bonferroni-corrected paired Wilcoxon tests (App.~\ref{app:per_domain}) confirm the effect grows with dimension and is not attributable to chance. JANUS-GN is the recommended default when CMIX's inverse Hessian estimate is noisy.

\paragraph{Comparison with model-based DFO and host-agnostic validation.}
Against model-based derivative-free optimizers (NEWUOA, BOBYQA, COBYLA~\citep{powell2006newuoa,powell2009bobyqa}, DFO-LS~\citep{cartis2019dfols}) at $d{=}30$/$100$ (30 seeds; App.~\ref{app:baselines}, Table~\ref{tab:new_baselines}(a)), the rerun CMA-ES host has geometric-mean final values $0.8$/$20.1$, JANUS-GN improves these to $0.1$/$14.7$, and standalone DFO is far worse ($1351$--$2419$ at $d{=}30$, $2112$--$5477$ at $d{=}100$), leaving model construction costs behind the population-plus-local-geometry combination. To test whether the metric principle is CMA-specific, we also attach JANUS infill to a DE host (30 seeds, 16 functions, $d\in\{30,100,500\}$): JANUS-DE beats vanilla DE on $86\%$/$91\%$/$96\%$ of runs with geometric-mean ratios $0.86$/$0.58$/$0.59$ (Table~\ref{tab:new_baselines}(b)), confirming metric-guided infill is not CMA-ES-specific.

\paragraph{NN-BBO baselines (same-protocol, $d{=}500$).}
Table~\ref{tab:bbob_dense_500d} reports same-protocol comparisons against 8 NN-BBO / MetaBBO baselines: Random Search, GLHF~\citep{li2024pretrained}, SYMBOL~\citep{ICLR2024_b742b708}, LES~\citep{lange2023discovering}, RLEPSO~\citep{yin2021rlepso}, GLEET~\citep{10.1145/3638529.3653996}, LDE~\citep{sun2021lde}, SynCMA~\citep{zhang2024invariant}, and DiBO~\citep{yun2025posterior} (compact 4-baseline slice shown; full 10-baseline table in App.~\ref{app:bbob}, Table~\ref{tab:bbob_dense_500d_full}). B2Opt~\citep{Li_Wu_Zhang_Wang_2025} is omitted at $d{=}500$ because its policy network is fixed at $d{=}10$ and cannot accept 500-dimensional inputs. Full JANUS attains the best mean error on \textbf{13 of the 16} functions in the complete baseline comparison (App.~\ref{app:bbob}, Table~\ref{tab:bbob_dense_500d_full}; the compact slice below shows $9$ blue cells, as the JANUS-GN ablation column can hold the best cell and $f_{20}$ lacks baseline coverage). The geometric-mean final error is $248$ for JANUS, versus $8.6{\times}10^4$ for CMA-ES and $1.0{\times}10^4$ for the best NN-BBO baseline (LDE)---a $>40\times$ reduction over the strongest learned baseline. Among the external baselines the only per-function wins are $f_{24}$ (RLEPSO) and $f_{23}$ (SYMBOL), both narrow margins where the residual Jacobian carries little exploitable curvature.

\begin{table}[t]
\centering
\maintablesetup
\caption{UAV main comparison at 2,500 FEs ($d{=}1135$). Lower is better; full 12-method, 28-terrain table and path visualizations are in App.~\ref{app:uav}.}
\label{tab:uav_main_compact}
\begin{tabular*}{\linewidth}{@{\extracolsep{\fill}}lrrr}
\toprule
Method & Mean & Median & Rank-1 \\
\midrule
\textbf{JANUS} & \bestcell{9,600} & \bestcell{6,644} & \bestcell{15} \\
DE & \secondcell{11,006} & 9,261 & \secondcell{1} \\
JDE21 & 11,017 & \secondcell{8,579} & \secondcell{1} \\
RLDEAFL & 11,411 & 8,197 & 0 \\
CMA-ES & 14,486 & 10,908 & 0 \\
\bottomrule
\end{tabular*}
\end{table}

\begin{table}[t]
\centering
\tightmaintablesetup
\caption{BBOB $d{=}500$: mean final error, 15 held-out functions with baseline coverage ($[-5,5]^d$, 20,000 FEs; $f_{20}$ omitted, no baseline data). NN-BBO baselines and JANUS rerun under our protocol (10 seeds); JANUS-GN from the host-controlled run (30 seeds). Blue$=$best, green$=$second. Lower is better. Full 10-baseline comparison in App.~\ref{app:bbob}, Table~\ref{tab:bbob_dense_500d_full}. The \emph{Wins} row counts best (blue) cells \emph{within this reduced 7-column slice}, in which the ablation column JANUS-GN can hold the best cell; this differs from the ``13/16'' figure in the text, which counts functions where full JANUS is best among all baselines over the complete 16-function set. \emph{Sig.}: Bonferroni-corrected wins/ties/losses vs JANUS ($p<0.05$), by a paired Wilcoxon signed-rank test over the 15 per-function final errors (each function's mean over the first 10 JANUS seeds, paired against the corresponding baseline mean); the pairing unit is the function identity, so this is a cross-function comparison of aggregated final errors rather than a within-function seed test.}
\label{tab:bbob_dense_500d}
\begin{tabular*}{\textwidth}{@{\extracolsep{\fill}} c c cccc cc}
\toprule
 & \multicolumn{1}{c}{Host} & \multicolumn{4}{c}{\textbf{NN-BBO / MetaBBO baselines}} & \multicolumn{2}{c}{\textbf{Ours}} \\
\cmidrule(lr){2-2} \cmidrule(lr){3-6} \cmidrule(lr){7-8}
Func & CMA-ES & RS & SYMBOL & RLEPSO & LDE & \textbf{JANUS-GN} & \textbf{JANUS} \\
\midrule
$f_{4}$ & 4.2e4 & 3.7e5 & 1.2e5 & 2.5e4 & \secondcell{1.6e4} & \secondcell{1.6e4} & \bestcell{5992} \\
$f_{6}$ & 1.8e6 & 1.6e7 & 5.6e6 & 9.2e5 & 9.9e5 & \secondcell{7103} & \bestcell{2113} \\
$f_{7}$ & 1.5e4 & 3.9e4 & 1.8e4 & 1.2e4 & \secondcell{4824} & -- & \bestcell{988} \\
$f_{8}$ & 5.9e8 & 5.9e8 & 4.7e7 & 1.5e7 & 1.3e7 & \secondcell{3072} & \bestcell{1426} \\
$f_{9}$ & 6.1e8 & 4.1e8 & 2.4e6 & 9.5e5 & 1.8e5 & \bestcell{508.7} & \secondcell{1323} \\
$f_{10}$ & 3.3e7 & 2.8e8 & 8.9e7 & 1.8e7 & 1.1e7 & \bestcell{1799} & \secondcell{1.5e6} \\
$f_{11}$ & 1.2e4 & 5833 & 3099 & \secondcell{2715} & 3141 & \bestcell{1934} & 4037 \\
$f_{12}$ & 8.2e9 & 3.0e10 & 4.8e9 & 1.9e9 & 1.5e9 & \secondcell{1.8e6} & \bestcell{1.0e4} \\
$f_{13}$ & 1.1e4 & 1.5e4 & 8913 & 5772 & 5347 & \secondcell{4728} & \bestcell{88.6} \\
$f_{14}$ & 267 & 655 & 176 & 82.5 & \secondcell{73.0} & 1.5e5 & \bestcell{0.12} \\
$f_{18}$ & 73.9 & 140 & 70.3 & 54.7 & \secondcell{40.6} & 157.6 & \bestcell{14.5} \\
$f_{19}$ & 3006 & 2096 & 24.2 & 22.4 & 12.3 & \bestcell{0.34} & \secondcell{11.7} \\
$f_{22}$ & 84.1 & 86.4 & 81.2 & 63.6 & 62.2 & \secondcell{11.6} & \bestcell{2.87} \\
$f_{23}$ & 1.65 & 1.64 & \bestcell{1.46} & 1.65 & 1.64 & 11.1 & \secondcell{1.66} \\
$f_{24}$ & 2.3e4 & 2.1e4 & 9315 & \bestcell{7426} & 7660 & 7448 & \secondcell{7435} \\
\midrule
\textbf{Wins} & 0 & 0 & 1 & 1 & 0 & 5 & \textbf{9} \\
\textbf{Sig. vs JANUS} & 15/1/0 & 14/2/0 & -- & -- & -- & 0/16/0 & -- \\
\bottomrule
\end{tabular*}
\end{table}

\begin{figure}[t]
\centering
\includegraphics[width=0.9\textwidth]{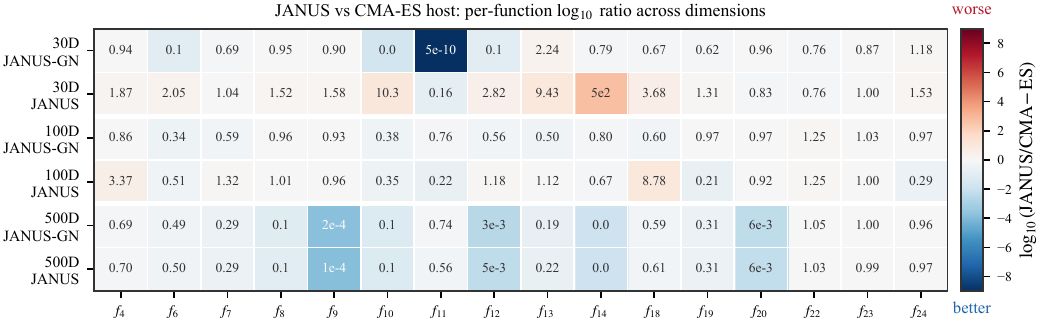}
\caption{Per-function BBOB host comparison across $d\in\{30,100,500\}$. Color encodes $\log_{10}(\text{JANUS variant}/\text{CMA-ES host})$ under the same protocol; blue$=$JANUS better, red$=$worse.}
\label{fig:bbob_heatmap}
\end{figure}

\subsection{High-Dimensional UAV Path Planning ($d{=}1135$)}
\label{sec:exp-uav}

We next test whether the same local-geometry reuse helps in a high-dimensional structured application: the UAV problem optimizes a 1135-dimensional spherical-coordinate flight path under terrain-dependent costs, and JANUS estimates residual coordinates from the observed optimization trace rather than requiring analytic derivatives. We evaluate on 28 held-out terrains with 30 runs each (840 final values per method) under a short 2,500-FE budget. Baselines span traditional optimizers (DE, PSO, RS), adaptive variants (JDE21~\citep{brest2021self}, CMA-ES, SAHLPSO~\citep{tao2021sahlpso}), and external MetaBBO methods (RLDEAFL~\citep{10.1145/3712256.3726309}, GLEET, GLHF, LES).

Table~\ref{tab:uav_main_compact} shows that JANUS achieves the best mean cost (9,600) and median (6,644), $12.8\%$ below the strongest external baseline (DE at 11,006); JDE21 has the tightest IQR among the top methods (2,795 vs JANUS's 9,579), though at a higher mean and median. The observed UAV gains are consistent with a regime in which the local residual coordinates are more informative and stable than the scalar-derived coordinates used on BBOB, so inverse Hessian exploration is more reliable under the short budget.

The terrain-level results are heterogeneous: JANUS ranks first on 15/28 terrains (best among all 12 compared methods, including JANUS-GN; App.~\ref{app:uav} reports the complementary count against external baselines only), while other methods lead the remaining terrains under the short budget, typically where early corridor discovery matters more than local refinement. UAV exposes more meaningful residual structure than scalar BBOB, but the terrain distribution spans both regimes, so App.~\ref{app:uav} reports the full per-terrain audit rather than only the aggregate mean. BO-like sanity checks (sequential ET-BO, RF-SMAC-like~\citep{hutter2011smac}, TuRBO-like~\citep{eriksson2019turbo} batched) on a terrain subset give mean costs of $13{,}900$--$14{,}600$ (App.~\ref{app:baselines}, Table~\ref{tab:new_baselines}(c)), substantially worse than JANUS ($9{,}600$), consistent with the difficulty of surrogate modeling at $d{=}1135$ with only $2{,}500$ evaluations.

\textbf{Long-budget validation.} Extending the same terrain protocol to $10{,}000$ and $20{,}000$ FEs (App.~\ref{app:uav}, Table~\ref{tab:uav_long_budget_main}), JANUS still leads at $10$k ($8{,}488$ vs CMA-ES $8{,}587$, $-1.2\%$); at $20$k the methods converge to near-parity ($8{,}019$ vs $8{,}023$). Thus JANUS is best interpreted as an early-budget accelerator — the regime of interest when evaluations cost minutes to hours — and the 20,000-FE result suggests that JANUS does not cause a persistent degradation in this UAV setting.

\begin{figure}[t]
\centering
\begin{subfigure}{0.22\textwidth}\centering
\includegraphics[width=\linewidth]{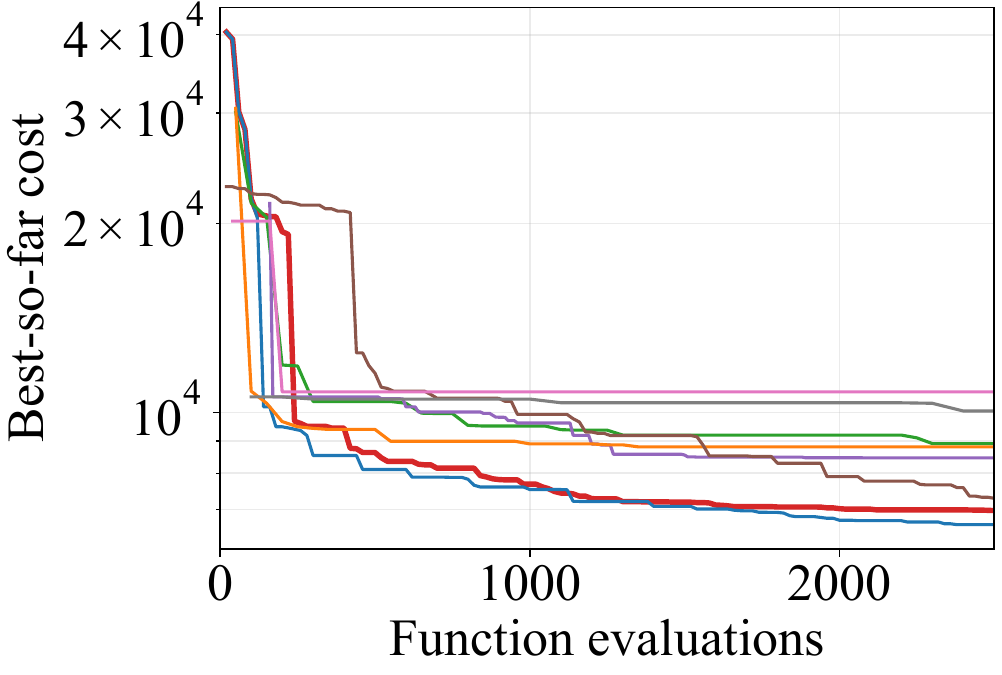}
\caption{T1 convergence}
\end{subfigure}\hfill
\begin{subfigure}{0.22\textwidth}\centering
\includegraphics[width=\linewidth]{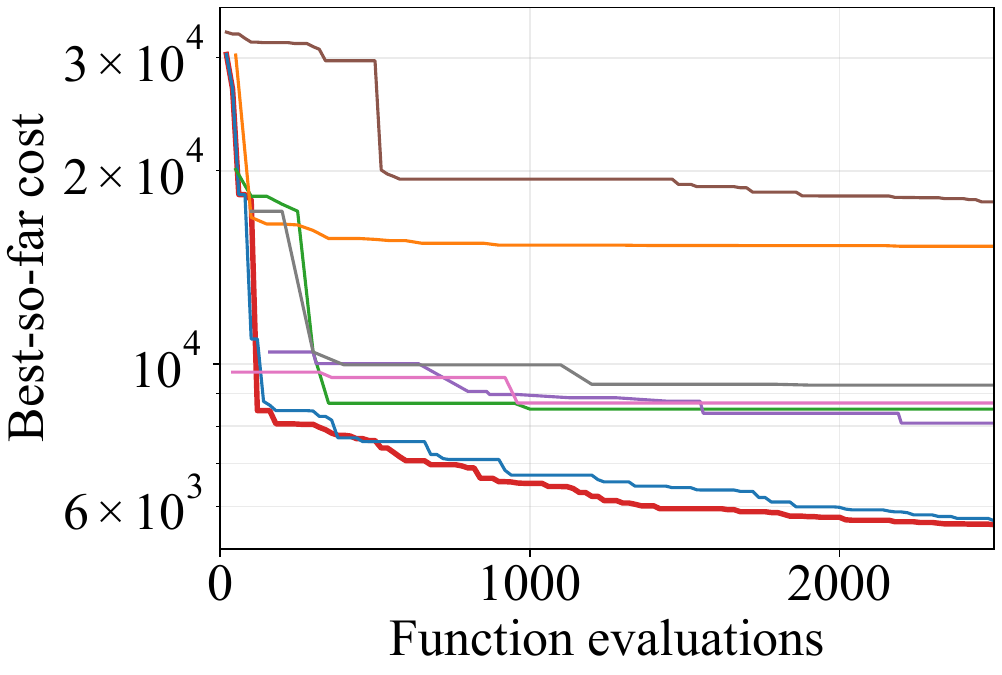}
\caption{T17 convergence}
\end{subfigure}\hfill
\begin{subfigure}{0.22\textwidth}\centering
\includegraphics[width=\linewidth]{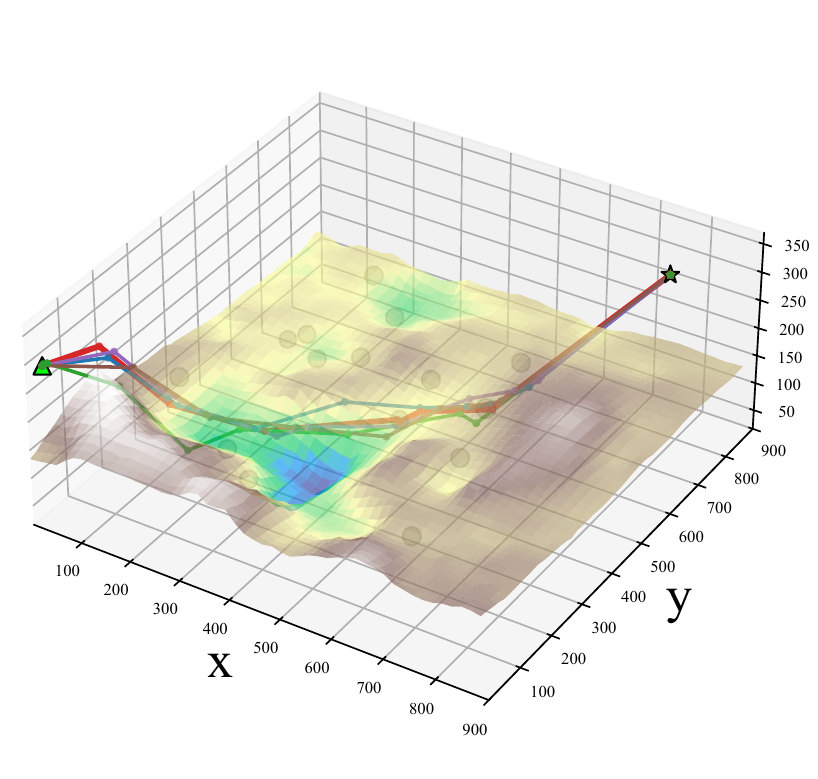}
\caption{T1 paths}
\end{subfigure}\hfill
\begin{subfigure}{0.22\textwidth}\centering
\includegraphics[width=\linewidth]{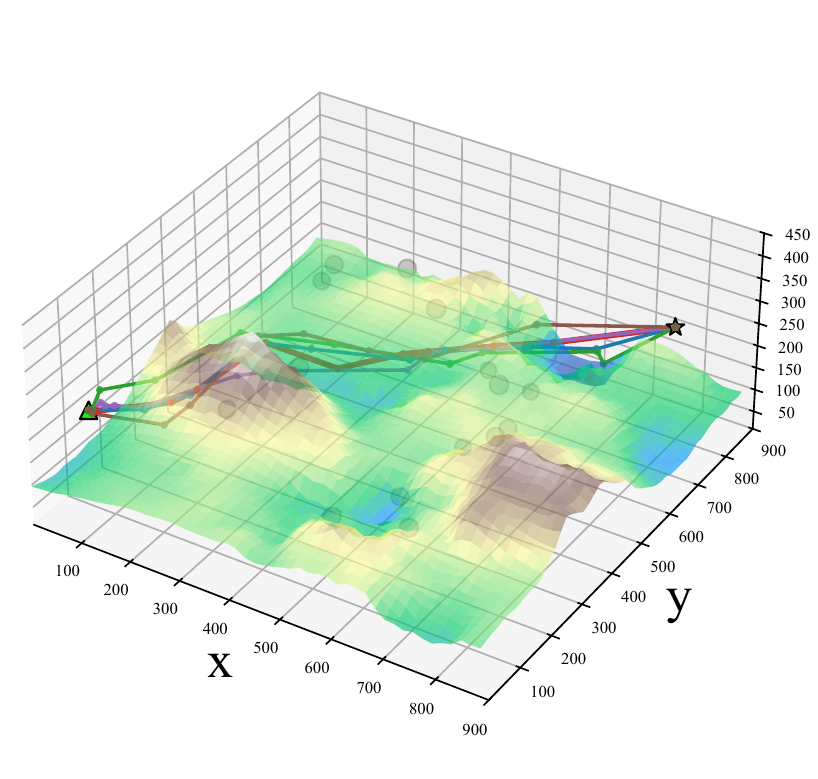}
\caption{T17 paths}
\end{subfigure}\hfill
\begin{subfigure}{0.08\textwidth}\centering
\includegraphics[width=\linewidth]{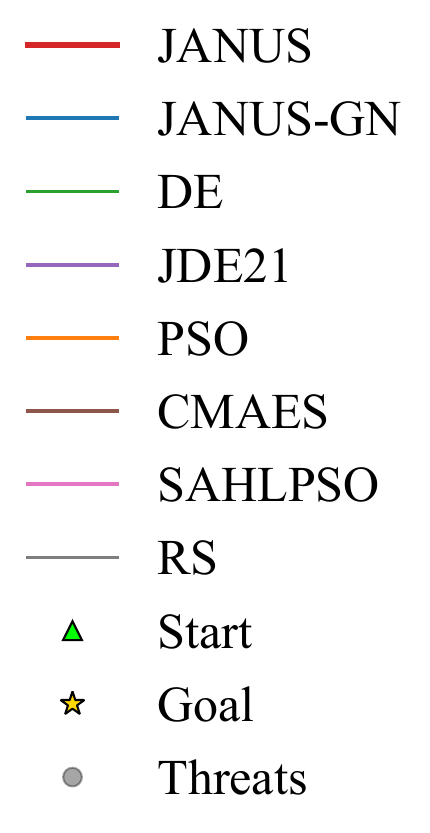}
\caption*{}
\end{subfigure}
\caption{\textbf{UAV convergence and path quality} on two representative terrains. JANUS separates early and produces smooth low-cost paths; JANUS-GN follows more conservatively, while other baselines often plateau higher under the 2,500-FE budget.}
\label{fig:uav_convergence}
\end{figure}

Figure~\ref{fig:uav_convergence} shows the mechanism behind the UAV aggregate numbers: the convergence panels show JANUS separating from CMA-ES and the traditional baselines before the full budget is exhausted, while the 3D path panels show the corresponding effect in the native trajectory space—locally informed candidates steer the host toward smoother, lower-cost corridors rather than merely improving a scalar score at the end of the run.

JANUS-GN follows the same early descent trend but more conservatively. The full JANUS variant gains more when CMIX has a stable local metric to exploit; JANUS-GN wins several individual terrains in App.~\ref{app:uav} where corridor discovery matters more than metric reshaping under the short budget.

\subsection{Multi-Objective Optimization (38 tasks)}
\label{sec:moo}

We then test whether JANUS can improve proposal generation without interfering with Pareto survival, treating each objective as a residual coordinate and generating a small fraction ($\rho{=}5\%$) of offspring from the inverse Hessian metric while leaving host survival unchanged. For the low-objective suite (18 tasks, 2--3 objectives over ZDT, DTLZ, WFG), we pair JANUS with SMS-EMOA (hypervolume-based survival); for the many-objective suite (20 tasks, 5/8/10 objectives), we pair JANUS with AGE-MOEA2 (geometry-aware survival). All experiments use 20,000 evaluations, 10 runs, and unified post-hoc hypervolume reference points.

\paragraph{Headline: 12/38 significant wins, 0 regressions, regime-dependent advantage.}
Table~\ref{tab:moo_leaderboard} summarizes the leaderboard. JANUS-SMS 5\% and JANUS-AGE2 5\% each achieve the best average rank (2.00) in their respective regime. Pairwise Mann--Whitney $U$ tests ($p<0.05$) against the strongest host baseline show 3/18 and 9/20 significant wins respectively (12/38 total; App.~\ref{app:moo} additionally reports raw best-average-rank win counts, higher because they are not gated by a significance test). No statistically significant regressions were detected under these tests and the Bonferroni correction --- a property that matters more than raw win rate here, since infill mechanisms that reshape proposal distributions risk degrading the host where geometry is unhelpful. The remaining tasks are statistically tied.

\begin{table}[t]
\centering
\tightmaintablesetup
\caption{MOO main results. 20,000 evaluations, 10 runs/task (180 low-, 200 many-objective), unified post-hoc HV reference points. Blue cell$=$best, green cell$=$second per regime. $+/{\approx}/-$: Mann--Whitney significant wins/ties/losses vs strongest host.}
\label{tab:moo_leaderboard}
\begin{tabular*}{\linewidth}{@{\extracolsep{\fill}} llrrl}
\toprule
Regime & Method & Avg.\ Rank & Norm.\ HV & $+/{\approx}/-$ \\
\midrule
Low & \textbf{JANUS-SMS 5\%} & \bestcell{2.00} & \bestcell{0.994} & 3/15/0 \\
(2--3 obj.) & JANUS-SMS 2\% & \secondcell{2.50} & \secondcell{0.993} & -- \\
 & SMS-EMOA & 3.11 & 0.972 & -- \\
 & NSGA-II & 3.61 & 0.967 & -- \\
 & MOEA/D & 5.00 & 0.926 & -- \\
 & Random & 6.50 & 0.756 & -- \\
 & CMA-ES & 8.11 & 0.711 & -- \\
\midrule
Many & \textbf{JANUS-AGE2 5\%} & \bestcell{2.00} & \secondcell{0.972} & 9/11/0 \\
(5/8/10 obj.) & JANUS-AGE2 2\% & \secondcell{2.30} & \bestcell{0.973} & -- \\
 & AGE-MOEA2 & 3.65 & 0.915 & -- \\
 & NSGA-III & 5.00 & 0.888 & -- \\
 & RVEA & 6.40 & 0.865 & -- \\
 & NSGA-II & 6.50 & 0.917 & -- \\
 & MOEA/D & 7.45 & 0.820 & -- \\
\bottomrule
\end{tabular*}
\end{table}

\paragraph{Regime-dependent gains.}
The MOO results are intentionally reported by regime rather than pooled into a single headline; the rank matrix (App.~\ref{app:moo}, Fig.~\ref{fig:moo_matrix}) localizes the advantage. JANUS-SMS is strongest on WFG/DTLZ, where a few metric-guided candidates quickly improve the hypervolume-dominated region; JANUS-AGE2 leads at 5/8 objectives, where objective-vector residuals expose useful local geometry that geometry-aware survival exploits. At 10 objectives the gap narrows for both hosts because survival is already strong, mirroring the scalar BBOB boundary: JANUS helps when candidate geometry is a bottleneck, and its effect shrinks once host survival has saturated the diversity-convergence trade-off. JANUS is thus most valuable as a default-on augmentation for hosts facing moderate-dimensional ($m\le10$) Pareto fronts, where a small infill fraction improves convergence at negligible risk. \emph{Runtime}: with CUDA-accelerated metric estimation, JANUS-SMS/AGE2 5\% are $1.46\times$/$1.22\times$ slower than SMS-EMOA/AGE-MOEA2 (App.~\ref{app:moo}, Fig.~\ref{fig:ablation_panel})—unlikely to be the bottleneck when evaluation dominates wall-clock time.

\subsection{Ablation: Two Uses of One Jacobian}
\label{sec:exp-rq4}

Finally, we isolate the two uses of $J$ via a component ablation over representative scalar, UAV, and MOO settings (Table~\ref{tab:rq4_strict}). GN improves early exploitation, CMIX is strongest when curvature is stable enough to guide exploration, and full JANUS gives the best scalar result while remaining competitive on structured and multi-objective tasks.

\begin{table}[t]
\centering
\maintablesetup
\setlength{\tabcolsep}{3.0pt}
\caption{Component ablation across representative domains. Ratios are relative to the base host; lower is better for BBOB/UAV and higher is better for MOO.}
\label{tab:rq4_strict}
\begin{tabular*}{\linewidth}{@{\extracolsep{\fill}} lccrrr}
\toprule
Variant & GN & CMIX & BBOB $\downarrow$ & UAV $\downarrow$ & MOO $\uparrow$ \\
\midrule
Base host &  &  & 1.00$\times$ & 1.00$\times$ & 1.00$\times$ \\
GN only & \cmark &  & \secondcell{0.85$\times$} & 0.91$\times$ & \bestcell{1.02$\times$} \\
CMIX only &  & \cmark & 0.92$\times$ & \bestcell{0.88$\times$} & 0.97$\times$ \\
\textbf{JANUS} & \cmark & \cmark & \bestcell{0.83$\times$} & \secondcell{0.89$\times$} & \secondcell{1.01$\times$} \\
\bottomrule
\end{tabular*}
\end{table}

\emph{Practical recommendation}: use JANUS (CMIX+GN, tuned fuse cadence) as the scalar default; fall back to JANUS-GN when the trace is noisy, curvature unstable, or the objective multimodal.

\subsection{Implementation and Protocol Controls}
\label{sec:implementation}

JANUS is a candidate generator around the host ask/tell interface. At each generation the host proposes its ordinary offspring; JANUS reads the most recent evaluated window, fits the residual chart, and contributes a few additional candidates, all evaluated before the host's standard tell/update step. It never overwrites the host covariance, restart logic, archive, or survival operator, so the implementation is portable across CMA-ES, DE, SMS-EMOA, and AGE-MOEA2-style loops. In CMA-style hosts we clip the proposed displacement in host-whitened coordinates to $\|\cdot\|\le\sqrt d$; DE and MOO hosts use analogous box and step-length constraints, keeping candidates on the host sampling scale and preventing an ill-conditioned Jacobian from dominating adaptation or survival. The chart is rebuilt more often for scalar objectives, where pseudo-residuals are less stable than observed components; for UAV/MOO, coordinates are observed directly and CMIX is used more aggressively (App.~\ref{app:hparams}). \paragraph{Operating regime.}
JANUS helps most when the recent trace holds a stable local metric the host does not already encode: the effect grows with dimension on BBOB, transfers to a proprietary 1135D ISP task ($22\%$ error reduction vs.\ CMA-ES at 5,000 evaluations, App.~\ref{app:awb}), and is strongest on UAV, where the objective decomposition yields meaningful residual coordinates. The same evidence marks its limits: JANUS-GN is safer on multimodal functions, short UAV budgets can leave curvature unstable before CMIX has enough trace, and in MOO the gain is largest before host survival saturates diversity (App.~\ref{app:negative}). All comparisons share seeds, budgets, initialization, and protocol (App.~\ref{app:protocols}).

\section{Conclusion}
\label{sec:conclusion}

Population optimizers discard the local geometry already in their evaluation trace; JANUS recovers it through one online residual Jacobian, fed back as bounded infill without touching host selection, covariance adaptation, or step-size control. With hyperparameters shared across tasks (bar the BBOB fuse cadence; App.~\ref{app:hparams}), the same mechanism improves CMA-ES on 11--15/16 BBOB functions, attains the best mean UAV cost, and yields 12/38 significant MOO wins with no significant regressions. Under the evaluated protocols this pairing outperformed either standalone variant.

\FloatBarrier
\clearpage
\bibliographystyle{iclr2027_conference}
\bibliography{references}

\clearpage
\begin{center}
{\LARGE\bfseries JANUS: Online Jacobian-Aligned Infill\par}
{\large Supplementary Material\par}
\end{center}
\vspace{0.5em}
\suppressfloats[t]
\appendix
\section{Hyperparameters and Reproducibility}
\label{app:hparams}

Table~\ref{tab:hparams} lists the JANUS hyperparameters held fixed in all experiments. The adaptive mixing weight $w$ (Eq.~\ref{eq:weight}) is derived only from problem dimensions. One exception is the BBOB fuse cadence: fuse\_every is tuned once for the $d\in\{30,100\}$ scalar setting (fuse\_every=50) rather than reused verbatim at $d=500$; every other hyperparameter in Table~\ref{tab:hparams} is shared unchanged across every experiment in this paper.

For completeness, the protocols used by the main paper are: BBOB uses 16 held-out functions, $d\in\{30,100,500\}$, 20,000 FEs, and 30 runs; UAV uses 28 held-out terrains, 2,500 FEs, and 30 runs; MOO uses pymoo hosts, 20,000 evaluations, 10 runs, and unified post-hoc hypervolume reference points.

\subsection{Search Spaces and Evaluation Protocols}
\label{app:protocols}

\paragraph{BBOB.} We use a MetaBox/COCO-compatible protocol: 16 held-out functions, $d\in\{30,100,500\}$, $[-5,5]^d$, 20,000 FEs, and 30 runs. The CMA-ES host comparison in Fig.~\ref{fig:bbob_heatmap} matches functions, dimensions, seeds, population size, initialization, box constraints, and budget. Non-ABOM baselines use this setting whenever supported. ABOM is treated only as a reported reference because its paper states a wider box and does not provide public source code.

\paragraph{UAV path planning.} UAV uses a 1135-dimensional spherical-coordinate path parameterization decoded to Cartesian trajectories before evaluation. We use 28 held-out terrains, 30 runs, and a 2,500-FE budget; the long-budget check repeats the protocol at 10,000 and 20,000 FEs. Table~\ref{tab:dense_uav_terrain_panel} keeps per-terrain results because the terrain distribution is heterogeneous.

\paragraph{Multi-objective optimization.} MOO experiments use pymoo tasks and host survival operators: SMS-EMOA-style selection for low-objective tasks and AGE-MOEA2~\citep{panichella2022agemoea2}-style survival for many-objective tasks. Each task uses 20,000 evaluations and 10 runs. Hypervolume uses unified post-hoc reference points, and JANUS uses the objective vector as the residual coordinate system.

\begin{table}[H]
\centering
\small
\renewcommand{\arraystretch}{1.15}
\caption{JANUS hyperparameters shared across all experiments. The sole per-setting exception is the BBOB fuse cadence (\emph{assembly period}), tuned once for the scalar BBOB setting ($d\in\{30,100\}$) and reused verbatim elsewhere.}
\label{tab:hparams}
\begin{tabularx}{\textwidth}{l l X}
\toprule
Symbol & Value & Meaning \\
\midrule
window & 300 & Sliding window of recent $(x,s)$ pairs for Jacobian regression \\
$\mu$ & $10^{-3}$ & Ridge regularization for residual-gradient regression \\
$\lambda$ & $[1,10]$ & Levenberg--Marquardt damping (stable across range) \\
$\kappa$ & 1.0 & Bounded-injection clip factor \\
$w_{\text{base}}$ & 0.5 & Base mixing weight before dimension scaling \\
$[w_{\min},w_{\max}]$ & $[0.02,0.6]$ & Clip range for the adaptive mixing weight \\
assembly period & 4--6 gen$^\dagger$ & Generations between Jacobian re-estimation ($^\dagger$50 for the scalar BBOB setting) \\
$\rho$ (MOO) & 0.05 & Fraction of offspring allocated to JANUS infill \\
\bottomrule
\end{tabularx}
\end{table}

\paragraph{Sensitivity sweep.} We vary each hyperparameter in Table~\ref{tab:hparams} one at a time (window size, damping $\lambda$, CMIX damping $\lambda_{\mathrm{mix}}$, infill ratio $\rho$, and the mixing-weight bounds $[w_{\min},w_{\max}]$) across 18 configurations on the $d{=}100$ BBOB suite. Most configurations stay within $1.5\times$ of the default geometric-mean final error, and no single setting collapses the method; the largest sensitivity is to the window size, since a window that is too short destabilizes the ridge-regression Jacobian estimate and a window that is too long mixes stale curvature into $J$. This supports using one fixed hyperparameter set across BBOB, UAV, and MOO rather than tuning per benchmark.
\FloatBarrier

\subsection{Notation and Variants}
\label{app:notation}
See Table~\ref{tab:notation} for notation and Table~\ref{tab:janus_variants} for implementation variants.
\begin{table}[H]
\centering
\wideappendixtablesetup
\caption{Notation and JANUS instantiations.}
\label{tab:notation}
\begin{tabular*}{\linewidth}{@{\extracolsep{\fill}} p{0.14\linewidth} p{0.52\linewidth} p{0.22\linewidth}}
\toprule
\multicolumn{3}{l}{\textbf{Notation}} \\
\midrule
Symbol & Meaning & Where used \\
\midrule
$x\in\mathcal{X}$ & Candidate solution in the native search space & All benchmarks \\
$f(x)$ & Scalar black-box objective or scalarized cost & BBOB/UAV \\
$r(x)$ & Observed component/objective vector when available & Structured/MOO tasks \\
$s(x)$ & Residual coordinates used for local Jacobian fitting & All \\
$\mathcal{W}$ & Recent evaluated window for local fitting & Online adaptation \\
$J$ & Local residual Jacobian from current window & Metric construction \\
$G=J^\top J$ & Normal metric reused for exploitation and exploration & JANUS infill \\
$\lambda$ & Damping/regularization & Numerical stability \\
$\rho$ & Fraction of candidates allocated to CMIX & Exploration budget \\
\midrule
\multicolumn{3}{l}{\textbf{Instantiations}} \\
\midrule
Name & Mechanism & Role in experiments \\
\midrule
JANUS-GN & Bounded damped Gauss--Newton infill from $G$ & Exploitation-only mode \\
JANUS & JANUS-GN plus inverse-curvature candidate mixing (CMIX) & Full scalar/UAV mode \\
JANUS-SMS & JANUS infill attached to SMS-EMOA survival & Low-objective MOO host \\
JANUS-AGE2 & JANUS infill attached to AGE-MOEA2 survival & Many-objective MOO host \\
JANUS-DE & JANUS infill attached to differential evolution & Non-CMA host check \\
Random infill & Same infill budget without $G$ & Metric-free control \\
\bottomrule
\end{tabular*}
\vspace{2pt}
\parbox{\linewidth}{\emph{Default host.} Scalar BBOB and UAV use CMA-ES unless otherwise stated; JANUS-GN and JANUS differ only by whether CMIX candidates are enabled. MOO variants use the host named in the method label.}
\label{tab:janus_variants}
\end{table}

\FloatBarrier

\subsection{Extended Baseline Comparisons}
\label{app:baselines}

Table~\ref{tab:new_baselines} (referenced from the main paper, Sec.~\ref{sec:exp-bbob} and Sec.~\ref{sec:exp-uav}) adds three checks: classical local modeling on BBOB, JANUS-DE as a non-CMA host check, and BO-style surrogate modeling on 1135-dimensional UAV.

\begin{table}[t]
\centering
\footnotesize\renewcommand{\arraystretch}{1.0}\setlength{\tabcolsep}{4pt}
\setlength{\abovecaptionskip}{2pt}
\setlength{\belowcaptionskip}{2pt}
\caption{Extended baseline checks (Sec.~\ref{sec:exp-bbob}). Lower is better except win rate. Blue cell$=$best, green cell$=$second.}
\label{tab:new_baselines}

\textbf{(a) BBOB DFO comparison, geometric-mean final value.}\par\vspace{2pt}
\begin{tabular*}{\linewidth}{@{\extracolsep{\fill}} lrr}
\toprule
Method & $d{=}30$ & $d{=}100$ \\
\midrule
\textbf{JANUS-GN} & \bestcell{0.1} & \bestcell{14.7} \\
CMA-ES host & \secondcell{0.8} & \secondcell{20.1} \\
BOBYQA & 1351 & 2112 \\
COBYLA & 1447 & 2524 \\
DFO-LS & 1864 & 2641 \\
NEWUOA & 2419 & 5477 \\
\bottomrule
\end{tabular*}

\vspace{5pt}
\textbf{(b) JANUS-DE vs.\ vanilla DE, host-agnostic check.}\par\vspace{2pt}
\begin{tabular*}{\linewidth}{@{\extracolsep{\fill}} lrrr}
\toprule
Metric & $d{=}30$ & $d{=}100$ & $d{=}500$ \\
\midrule
\textbf{JANUS-DE win rate} & 86\% & 91\% & 96\% \\
\textbf{JANUS-DE geo ratio} & 0.86 & 0.58 & 0.59 \\
\bottomrule
\end{tabular*}

\vspace{5pt}
\textbf{(c) UAV BO-like sanity check, $d{=}1135$, 2,500 FEs, 3 terrains.}\par\vspace{2pt}
\begin{tabular*}{\linewidth}{@{\extracolsep{\fill}} lr}
\toprule
Method & Mean cost \\
\midrule
\textbf{JANUS} & \bestcell{9,600} \\
ET-BO & \secondcell{13,923} \\
RF-SMAC & 14,022 \\
TuRBO & 14,634 \\
Random & 17,701 \\
\bottomrule
\end{tabular*}
\end{table}

\paragraph{Restricted-covariance baselines (LM-CMA, VkD-CMA).}
\label{app:restricted-cov}
Table~\ref{tab:restricted_cov_baselines} isolates the structural prediction of Prop.~\ref{prop:restricted} by comparing two restricted-covariance hosts---LM-CMA (LPCMA)~\citep{loshchilov2014lm} and VkD-CMA (VLPCMA)~\citep{akimoto2014vkd}---against a full-rank CMA-ES host on the 16 held-out BBOB functions at $d\in\{30,100\}$, 10 seeds, 5,000 FEs, all under this single shared protocol. It reports geometric-mean final \emph{error} (value minus the known optimum) at 5,000 FEs over 10 seeds, and is therefore separate from the DFO panel of Table~\ref{tab:new_baselines}, which reports geometric-mean final \emph{value} at 20,000 FEs over 30 seeds. LPCMA and VLPCMA are the canonical restricted-covariance baselines for Prop.~\ref{prop:restricted}: both adapt covariance in a fixed low-dimensional basis, so the proposition predicts they should underperform a full-rank host whenever the residual Jacobian carries curvature outside that basis. Empirically, LPCMA achieves geometric-mean $67.4$ ($d{=}30$) and $4.98{\times}10^3$ ($d{=}100$) and VLPCMA achieves $2.33{\times}10^3$ and $1.96{\times}10^4$; both are worse than the full-rank CMA-ES host ($14.8$ and $332$) at every dimension, and the gap widens from $4.6\times$ at 30D to $15.0\times$ at 100D for LPCMA. The observed ordering is compatible with the structural limitation described by Prop.~\ref{prop:restricted} --- restricted-covariance hosts cannot express an arbitrary residual-aligned metric --- although the proposition addresses expressiveness and does not by itself predict optimization performance, the direction of the gap, or its growth with dimension; those remain empirical observations. (JANUS's own advantage over a full-rank host is a separate, convergence-speed claim, quantified on real 30-seed data in the DFO panel of Table~\ref{tab:new_baselines}, not here.)

\begin{table}[H]
\centering
\wideappendixtablesetup
\caption{BBOB restricted- vs.\ full-rank host comparison: 16 held-out functions, 10 seeds, 5,000 FEs, $[-5,5]^d$. Geometric-mean final error across 16 functions, and per-function wins (count of functions on which each method attains the lowest error). Both LPCMA (LM-CMA) and VLPCMA (VkD-CMA), which adapt covariance in a fixed low-dimensional basis, underperform the full-rank CMA-ES host at $d\in\{30,100\}$. This is compatible with the structural limitation of Prop.~\ref{prop:restricted} --- restricted-covariance hosts cannot express an arbitrary residual-aligned metric --- though the proposition concerns expressiveness and does not by itself predict this performance ordering. The full-rank CMA-ES host is $4.6\times$ better than LM-CMA at 30D and $15.0\times$ better at 100D. Blue cell$=$best, green cell$=$second-best.}
\label{tab:restricted_cov_baselines}
\begin{tabular*}{\linewidth}{@{\extracolsep{\fill}} lcccc}
\toprule
 & \multicolumn{2}{c}{$d{=}30$} & \multicolumn{2}{c}{$d{=}100$} \\
\cmidrule(lr){2-3} \cmidrule(lr){4-5}
Method & Geo-Mean & Wins & Geo-Mean & Wins \\
\midrule
CMA-ES (host, full-rank) & \bestcell{14.8} & \bestcell{16} & \bestcell{332} & \bestcell{16} \\
LM-CMA (LPCMA) & \secondcell{67.4} & 0 & \secondcell{4.98e3} & 0 \\
VkD-CMA (VLPCMA) & 2.33e3 & 0 & 1.96e4 & 0 \\
Random Search & 3.63e3 & 0 & 1.90e4 & 0 \\
\bottomrule
\end{tabular*}
\end{table}

\noindent\emph{Readout.} Classical DFO builds useful local scalar models, but these standalone methods spend many evaluations maintaining local interpolation or trust-region models and do not scale well under the tested BBOB budgets. This does not mean that local modeling is unhelpful; rather, JANUS uses local geometry as a bounded infill signal on top of a population host instead of replacing the host. JANUS-DE further shows that this signal is not specific to CMA-ES adaptation. The UAV BO-like sanity check is included for the same reason: at $d{=}1135$ and 2,500 evaluations, generic surrogates struggle to learn a useful global model, whereas JANUS only asks for a local metric from the recent trace.
\FloatBarrier

\section{Per-Domain Evidence}
\label{app:per_domain}

\subsection{BBOB Per-Function Results}
\label{app:bbob}

\paragraph{Takeaway.} BBOB is strong in aggregate, but the per-function tables are important because scalar BBOB is the least favorable setting for JANUS: the benchmark exposes only scalar values, so residual coordinates must be estimated from the recent trace. Table~\ref{tab:dense_bbob_function_strip} gives the compact per-function ratio view. The results therefore test whether pseudo-residual geometry is sufficient to improve a population optimizer, and where it becomes unreliable.

\begin{table}[H]
\centering
\scriptsize
\renewcommand{\arraystretch}{1.02}
\caption{Dense BBOB function strip. Entries are $\log_{10}(\mathrm{JANUS}/\mathrm{ABOM})$ final-value ratios; negative cells favor JANUS.}
\label{tab:dense_bbob_function_strip}
\begin{tabular*}{\linewidth}{@{\extracolsep{\fill}} lrrrrrl}
\toprule
Function & 30D & 100D & 500D & Wins & Mean log-ratio & Verdict \\
\midrule
F4 & \cellcolor{janusgreen!35}-0.75 & \cellcolor{janusgreen!35}-1.08 & \cellcolor{janusgreen!35}-0.23 & 3/3 & -0.69 & \cellcolor{janusgreen!55}robust win \\
F6 & \cellcolor{janusgreen!35}-9.39 & \cellcolor{janusgreen!35}-4.54 & \cellcolor{janusgreen!35}-0.30 & 3/3 & -4.74 & \cellcolor{janusgreen!55}robust win \\
F7 & \cellcolor{janusgreen!35}-1.84 & \cellcolor{janusgreen!35}-2.11 & \cellcolor{janusgreen!35}-0.67 & 3/3 & -1.54 & \cellcolor{janusgreen!55}robust win \\
F8 & \cellcolor{janusgreen!35}-1.53 & \cellcolor{janusgreen!35}-1.57 & \cellcolor{janusgreen!35}-1.66 & 3/3 & -1.59 & \cellcolor{janusgreen!55}robust win \\
F9 & \cellcolor{janusgreen!35}-3.25 & \cellcolor{janusgreen!35}-3.05 & \cellcolor{janusgreen!35}-2.04 & 3/3 & -2.78 & \cellcolor{janusgreen!55}robust win \\
F10 & \cellcolor{janusgreen!35}-3.30 & \cellcolor{janusgreen!35}-3.48 & \cellcolor{janusgreen!35}-0.52 & 3/3 & -2.43 & \cellcolor{janusgreen!55}robust win \\
F11 & \cellcolor{janusgreen!35}-3.59 & \cellcolor{janusgreen!35}-3.11 & \cellcolor{janusgreen!35}-0.18 & 3/3 & -2.29 & \cellcolor{janusgreen!55}robust win \\
F12 & \cellcolor{janusgreen!35}-7.41 & \cellcolor{janusgreen!35}-9.37 & \cellcolor{janusgreen!35}-2.06 & 3/3 & -6.28 & \cellcolor{janusgreen!55}robust win \\
F13 & \cellcolor{janusgreen!35}-1.40 & \cellcolor{janusgreen!35}-2.37 & \cellcolor{janusgreen!35}-0.82 & 3/3 & -1.53 & \cellcolor{janusgreen!55}robust win \\
F14 & \cellcolor{janusgreen!35}-2.39 & \cellcolor{janusgreen!35}-5.17 & \cellcolor{janusgreen!35}-1.60 & 3/3 & -3.05 & \cellcolor{janusgreen!55}robust win \\
F18 & \cellcolor{janusgreen!35}-3.28 & \cellcolor{janusgreen!35}-1.65 & \cellcolor{janusgreen!35}-0.34 & 3/3 & -1.75 & \cellcolor{janusgreen!55}robust win \\
F19 & \cellcolor{red!12}1.17 & \cellcolor{red!12}0.80 & \cellcolor{janusgreen!35}-0.20 & 1/3 & 0.59 & \cellcolor{janusamber!45}hard case \\
F20 & \cellcolor{red!12}ref<0 & \cellcolor{red!12}ref<0 & \cellcolor{janusgreen!35}-1.84 & 1/3 & -1.84 & \cellcolor{janusamber!45}hard case \\
F22 & \cellcolor{janusgreen!35}-0.93 & \cellcolor{janusgreen!35}-1.12 & \cellcolor{red!12}0.34 & 2/3 & -0.57 & \cellcolor{janusgreen!25}mostly win \\
F23 & \cellcolor{red!12}1.02 & \cellcolor{red!12}0.02 & \cellcolor{janusgreen!35}-0.00 & 1/3 & 0.35 & \cellcolor{janusamber!45}hard case \\
F24 & \cellcolor{janusgreen!35}-0.15 & \cellcolor{janusgreen!35}-0.59 & \cellcolor{janusgreen!35}-0.03 & 3/3 & -0.26 & \cellcolor{janusgreen!55}robust win \\
\bottomrule
\end{tabular*}
\end{table}

\noindent\emph{Readout.} Negative cells favor JANUS; red/amber cells mark the locality boundary. The full table below keeps the per-function mean$\pm$std audit at $d=30$ and $d=100$, because the variance matters for distinguishing systematic gains from unstable runs. With the tuned fuse cadence, JANUS (CMIX+GN on pyCMA host) wins 11/16 at $d{=}30$ and 15/16 at $d{=}100$; JANUS-GN provides a more conservative fallback that is competitive on per-function wins but has higher overall mean due to instability on a few ill-conditioned functions.
\begin{table}[H]
\centering
\wideappendixtablesetup
\setlength{\abovecaptionskip}{2pt}
\setlength{\belowcaptionskip}{2pt}
\caption{BBOB $d=30$ and $d=100$: mean$\pm$std; lower is better. 30D JANUS uses tuned configuration (fuse\_every=50, 10 seeds); JANUS-GN 30D and all 100D from the host-controlled run (30 seeds). Blue cell$=$best, green cell$=$second-best. \textbf{W/T/L}: wins/ties/losses vs CMA-ES per function (mean comparison).}
\label{tab:bbob_dense_30_100_std}
\resizebox{\linewidth}{!}{%
\begin{tabular*}{\linewidth}{@{\extracolsep{\fill}} cccc ccc}
\toprule
& \multicolumn{3}{c}{$d=30$} & \multicolumn{3}{c}{$d=100$} \\
\cmidrule(lr){2-4}\cmidrule(lr){5-7}
Func & \textbf{JANUS-GN} & \textbf{JANUS} & CMA-ES & \textbf{JANUS-GN} & \textbf{JANUS} & CMA-ES \\
\midrule
$f_4$ & \bestcell{\makecell{35.12\\ $\pm$7.39}} & \makecell{70.01\\ $\pm$14.79} & \secondcell{\makecell{66.59\\ $\pm$13.86}} & \bestcell{\makecell{979\\ $\pm$78.20}} & \secondcell{\makecell{1014\\ $\pm$74.55}} & \makecell{2685\\ $\pm$689} \\
$f_6$ & \makecell{2.44e-07\\ $\pm$3.8e-08} & \bestcell{\makecell{2.99e-08\\ $\pm$3.7e-08}} & \secondcell{\makecell{5.11e-08\\ $\pm$7.4e-08}} & \bestcell{\makecell{3140\\ $\pm$659}} & \secondcell{\makecell{3416\\ $\pm$1031}} & \makecell{1.80e+04\\ $\pm$2845} \\
$f_7$ & \bestcell{\makecell{0.276\\ $\pm$0.53}} & \secondcell{\makecell{8.36\\ $\pm$2.84}} & \makecell{9.82\\ $\pm$8.33} & \bestcell{\makecell{229\\ $\pm$84.01}} & \secondcell{\makecell{301\\ $\pm$95.92}} & \makecell{5101\\ $\pm$1841} \\
$f_8$ & \bestcell{\makecell{3.15\\ $\pm$2.26}} & \secondcell{\makecell{13.13\\ $\pm$1.60}} & \makecell{16.75\\ $\pm$12.93} & \secondcell{\makecell{2101\\ $\pm$3846}} & \bestcell{\makecell{1922\\ $\pm$3465}} & \makecell{3075\\ $\pm$4851} \\
$f_9$ & \bestcell{\makecell{1.71\\ $\pm$1.31}} & \secondcell{\makecell{11.50\\ $\pm$1.70}} & \makecell{11.69\\ $\pm$1.27} & \secondcell{\makecell{792\\ $\pm$1751}} & \bestcell{\makecell{670\\ $\pm$1817}} & \makecell{9.69e+04\\ $\pm$2.8e+05} \\
$f_{10}$ & \makecell{2.62e+04\\ $\pm$8.9e+04} & \secondcell{\makecell{134\\ $\pm$107}} & \bestcell{\makecell{122\\ $\pm$123}} & \bestcell{\makecell{2.06e+07\\ $\pm$4.3e+06}} & \secondcell{\makecell{3.28e+07\\ $\pm$7.0e+06}} & \makecell{1.08e+08\\ $\pm$2.3e+07} \\
$f_{11}$ & \bestcell{\makecell{9.18e-06\\ $\pm$1.1e-05}} & \secondcell{\makecell{38.79\\ $\pm$52.69}} & \makecell{60.02\\ $\pm$87.47} & \bestcell{\makecell{2.09e+05\\ $\pm$2.2e+04}} & \secondcell{\makecell{2.41e+05\\ $\pm$2.8e+04}} & \makecell{1.04e+06\\ $\pm$8.6e+05} \\
$f_{12}$ & \bestcell{\makecell{0.025\\ $\pm$0.044}} & \makecell{0.543\\ $\pm$1.08} & \secondcell{\makecell{0.0273\\ $\pm$0.075}} & \bestcell{\makecell{1.35e+05\\ $\pm$1.2e+05}} & \secondcell{\makecell{4.35e+05\\ $\pm$8.1e+05}} & \makecell{4.59e+08\\ $\pm$3.0e+08} \\
$f_{13}$ & \bestcell{\makecell{7.82e-03\\ $\pm$2.6e-03}} & \makecell{0.602\\ $\pm$0.85} & \secondcell{\makecell{0.498\\ $\pm$0.48}} & \bestcell{\makecell{6.66\\ $\pm$2.41}} & \secondcell{\makecell{7.09\\ $\pm$2.23}} & \makecell{105\\ $\pm$29.86} \\
$f_{14}$ & \makecell{2.35e-06\\ $\pm$4.7e-07} & \bestcell{\makecell{4.70e-07\\ $\pm$2.5e-07}} & \secondcell{\makecell{4.75e-07\\ $\pm$2.0e-07}} & \bestcell{\makecell{150\\ $\pm$39.76}} & \secondcell{\makecell{165\\ $\pm$42.95}} & \makecell{1376\\ $\pm$443} \\
$f_{18}$ & \makecell{1.33\\ $\pm$2.36} & \bestcell{\makecell{0.272\\ $\pm$0.35}} & \secondcell{\makecell{0.391\\ $\pm$0.24}} & \bestcell{\makecell{87.35\\ $\pm$32.66}} & \secondcell{\makecell{91.36\\ $\pm$28.72}} & \makecell{213\\ $\pm$45.14} \\
$f_{19}$ & \bestcell{\makecell{0.559\\ $\pm$0.85}} & \secondcell{\makecell{3.09\\ $\pm$1.74}} & \makecell{3.41\\ $\pm$2.01} & \bestcell{\makecell{7.86\\ $\pm$0.36}} & \secondcell{\makecell{7.93\\ $\pm$0.27}} & \makecell{19.63\\ $\pm$21.75} \\
$f_{20}$ & \bestcell{\makecell{-37.97\\ $\pm$4.61}} & \secondcell{\makecell{1.86\\ $\pm$0.2}} & \makecell{1.87\\ $\pm$0.19} & \secondcell{\makecell{-32.33\\ $\pm$2.58}} & \bestcell{\makecell{-33.10\\ $\pm$2.37}} & \makecell{-2.83\\ $\pm$1.20} \\
$f_{22}$ & \makecell{86.57} & \makecell{86.57} & \makecell{86.57} & \makecell{86.57} & \makecell{86.57} & \makecell{86.57} \\
$f_{23}$ & \bestcell{\makecell{2.93\\ $\pm$0.4}} & \secondcell{\makecell{3.13\\ $\pm$0.33}} & \makecell{3.23\\ $\pm$0.5} & \bestcell{\makecell{4.84\\ $\pm$0.34}} & \secondcell{\makecell{4.88\\ $\pm$0.34}} & \makecell{7.30\\ $\pm$0.57} \\
$f_24$ & \bestcell{\makecell{41.77\\ $\pm$8.40}} & \secondcell{\makecell{119\\ $\pm$79.99}} & \makecell{135\\ $\pm$65.53} & \secondcell{\makecell{888\\ $\pm$77.19}} & \bestcell{\makecell{882\\ $\pm$85.20}} & \makecell{1055\\ $\pm$27.55} \\
\midrule
\textbf{W/T/L vs CMA-ES} & 11/1/4 & 11/1/4 & -- & 15/1/0 & 15/1/0 & -- \\
\bottomrule
\end{tabular*}
}
\end{table}

\noindent\emph{Readout.} The dense tables match the strip plot. JANUS variants dominate well-conditioned and weakly multimodal functions, while losses concentrate on $f_{19}$/$f_{23}$-like landscapes where basin discovery or weak residual structure matters more than local conditioning. The host-controlled comparison shows the same locality boundary from a cleaner angle: at 500D, JANUS is behind or tied with the rerun CMA-ES host on $f_{22}$/$f_{23}$.

With the tuned fuse cadence, low-dimensional CMIX on the pyCMA host is competitive with JANUS-GN on per-function wins (11/16 at 30D). Without tuning, CMIX can overweight a noisy inverse Hessian estimate before the recent trace defines stable local directions; GN-only injection is then the safer fallback. ABOM uses a wider BBOB box and has no public source code, so ABOM ratios are reported-reference comparisons under an explicitly stated JANUS protocol, not same-code reproduction claims. The qualitative lesson is still useful: a local Jacobian can describe the current basin, but it cannot point to a disconnected one. These cases motivate restarts, host-level diversity, or conservative GN-only injection instead of simply increasing CMIX.

\noindent\emph{Practical implication.} The BBOB appendix supports JANUS (CMIX+GN, tuned fuse cadence) as the scalar default across $d\in\{30,100,500\}$; fall back to JANUS-GN when the trace is noisy, low-dimensional with unstable curvature, or visibly multimodal. This is also why we report both JANUS-GN and JANUS throughout the scalar tables rather than presenting only the full method.

\paragraph{Full $d{=}500$ NN-BBO/MetaBBO comparison.} The main paper's Table~\ref{tab:bbob_dense_500d} reports a compact 4-baseline slice (RS, SYMBOL, RLEPSO, LDE) of the same-protocol $d{=}500$ comparison for space. Table~\ref{tab:bbob_dense_500d_full} gives the full 10-baseline comparison, including GLHF, LES, GLEET, SynCMA, and DiBO. We use ``NN-BBO/MetaBBO'' as an umbrella label for these learned or configured optimizers, but they differ in mechanism: Random Search is a non-learning reference; RLEPSO and GLEET are reinforcement-learning-trained policies; LDE and GLHF are meta-trained/learned update rules; SYMBOL is a symbolically learned optimizer; LES is a learned evolution strategy; DiBO is a diffusion/posterior-based method; and SynCMA is a configured CMA variant. They are grouped only because all are trained or configured off-line and then run on the held-out functions without further adaptation; the comparison is same-protocol but not a claim that these methods share a training or generalization setting.

\begin{table}[t]
\centering
\wideappendixtablesetup
\caption{BBOB $d{=}500$, full 10-baseline comparison: mean final error, 16 held-out functions ($[-5,5]^d$, 20,000 FEs). NN-BBO baselines and JANUS rerun under our protocol (10 seeds). JANUS-GN shown from the host-controlled run (30 seeds). Blue cell$=$best, green cell$=$second. Lower is better. B2Opt: N/A (network fixed at 10D). \emph{Sig.} row: Bonferroni-corrected wins/ties/losses vs JANUS ($p<0.05$), by a paired Wilcoxon signed-rank test over the 16 per-function final errors (each function's mean over the first 10 JANUS seeds, paired against the corresponding baseline mean); the pairing unit is the function identity, so this is a cross-function comparison of aggregated final errors rather than a within-function seed test.}
\label{tab:bbob_dense_500d_full}
\resizebox{\textwidth}{!}{%
\begin{tabular}{c c ccccccccc cc}
\toprule
 & \multicolumn{1}{c}{Host} & \multicolumn{9}{c}{\textbf{NN-BBO / MetaBBO baselines}} & \multicolumn{2}{c}{\textbf{Ours}} \\
\cmidrule(lr){2-2} \cmidrule(lr){3-11} \cmidrule(lr){12-13}
Func & CMA-ES & RS & GLHF & SYMBOL & LES & RLEPSO & GLEET & LDE & SynCMA & DiBO & \textbf{JANUS-GN} & \textbf{JANUS} \\
\midrule
$f_{4}$ & 4.2e4 & 3.7e5 & 4.7e5 & 1.2e5 & 2.4e5 & 2.5e4 & 3.1e4 & \secondcell{1.6e4} & 3.4e4 & 2.1e5 & \secondcell{1.6e4} & \bestcell{5992} \\
$f_{6}$ & 1.8e6 & 1.6e7 & 1.8e7 & 5.6e6 & 9.3e6 & 9.2e5 & 1.5e6 & 9.9e5 & 5.8e5 & 9.8e6 & \secondcell{7103} & \bestcell{2113} \\
$f_{7}$ & 1.5e4 & 3.9e4 & 4.7e4 & 1.8e4 & 2.3e4 & 1.2e4 & 1.2e4 & \secondcell{4824} & 1.5e4 & 3.5e4 & -- & \bestcell{988} \\
$f_{8}$ & 5.9e8 & 5.9e8 & 7.1e8 & 4.7e7 & 5.1e7 & 1.5e7 & 3.1e7 & 1.3e7 & 9.5e7 & 5.1e8 & \secondcell{3072} & \bestcell{1426} \\
$f_{9}$ & 6.1e8 & 4.1e8 & 4.7e8 & 2.4e6 & 4543 & 9.5e5 & 4.4e6 & 1.8e5 & 7.0e6 & 4.6e8 & \bestcell{508.7} & \secondcell{1323} \\
$f_{10}$ & 3.3e7 & 2.8e8 & 3.7e8 & 8.9e7 & 2.1e8 & 1.8e7 & 2.5e7 & 1.1e7 & 4.1e7 & 2.2e8 & \bestcell{1799} & \secondcell{1.5e6} \\
$f_{11}$ & 1.2e4 & 5833 & 9115 & 3099 & 5136 & \secondcell{2715} & 2746 & 3141 & 3437 & 6575 & \bestcell{1934} & 4037 \\
$f_{12}$ & 8.2e9 & 3.0e10 & 4.8e10 & 4.8e9 & 1.1e10 & 1.9e9 & 2.5e9 & 1.5e9 & 5.5e9 & 1.9e10 & \secondcell{1.8e6} & \bestcell{1.0e4} \\
$f_{13}$ & 1.1e4 & 1.5e4 & 1.5e4 & 8913 & 1.3e4 & 5772 & 6709 & 5347 & 9840 & 1.3e4 & \secondcell{4728} & \bestcell{88.6} \\
$f_{14}$ & 267 & 655 & 838 & 176 & 1458 & 82.5 & 112 & \secondcell{73.0} & 252 & 633 & 1.5e5 & \bestcell{0.12} \\
$f_{18}$ & 73.9 & 140 & 171 & 70.3 & 371 & 54.7 & 57.9 & \secondcell{40.6} & 69.5 & 115 & 157.6 & \bestcell{14.5} \\
$f_{19}$ & 3006 & 2096 & 2488 & 24.2 & 2500 & 22.4 & 39.9 & 12.3 & 45.3 & 2392 & \bestcell{0.34} & \secondcell{11.7} \\
$f_{20}$ & -- & -- & -- & -- & -- & -- & -- & -- & -- & -- & \bestcell{16.9} & \secondcell{17.3} \\
$f_{22}$ & 84.1 & 86.4 & 86.5 & 81.2 & 1184 & 63.6 & 69.1 & 62.2 & 75.5 & 86.0 & \secondcell{11.6} & \bestcell{2.87} \\
$f_{23}$ & 1.65 & 1.64 & 1.85 & \bestcell{1.46} & 1202 & 1.65 & \secondcell{1.50} & 1.64 & 1.64 & 1.65 & 11.1 & 1.66 \\
$f_{24}$ & 2.3e4 & 2.1e4 & 2.2e4 & 9315 & 8878 & \bestcell{7426} & 8514 & 7660 & 9434 & 2.1e4 & 7448 & \secondcell{7435} \\
\midrule
\textbf{Wins} & 0 & 0 & 0 & 1 & 0 & 1 & 0 & 0 & 0 & 0 & 5 & \textbf{9} \\
\textbf{Sig. vs JANUS} & 15/1/0 & 14/2/0 & -- & -- & -- & -- & -- & -- & -- & -- & 0/16/0 & -- \\
\bottomrule
\end{tabular}%
}
\end{table}

\FloatBarrier

\paragraph{$d{=}1000$ extension.} To check whether the 500D trend continues at higher dimension, we evaluate JANUS-GN against the CMA-ES host on a five-function representative BBOB subset at $d{=}1000$, using the same $[-5,5]^d$ box, 20,000-FE budget, and 30 seeds as the main protocol. JANUS improves over the host on all five functions, reducing the geometric-mean final error from $1.19\times10^9$ to $1.27\times10^6$, a $936\times$ improvement. This is consistent with the 30D$\to$500D trend in Table~\ref{tab:bbob_dense_30_100_std}: as ambient dimension grows, host covariance adaptation from rank-based selection alone becomes harder, while the residual-Jacobian estimate remains a local, low-cost signal that does not need to scale with $d$ in the same way.
\FloatBarrier

\subsection{MOO Supplementary Results}
\label{app:moo}

\paragraph{Takeaway.} The MOO gains are regime-dependent rather than uniform. Unlike scalar BBOB, MOO provides natural objective vectors, so the residual coordinates are not synthetic. JANUS-SMS is strongest in low-objective hypervolume selection, while JANUS-AGE2 is strongest in the many-objective setting where AGE-style survival benefits from metric-guided candidates.

\begin{figure}[H]
\centering
\includegraphics[width=0.72\linewidth]{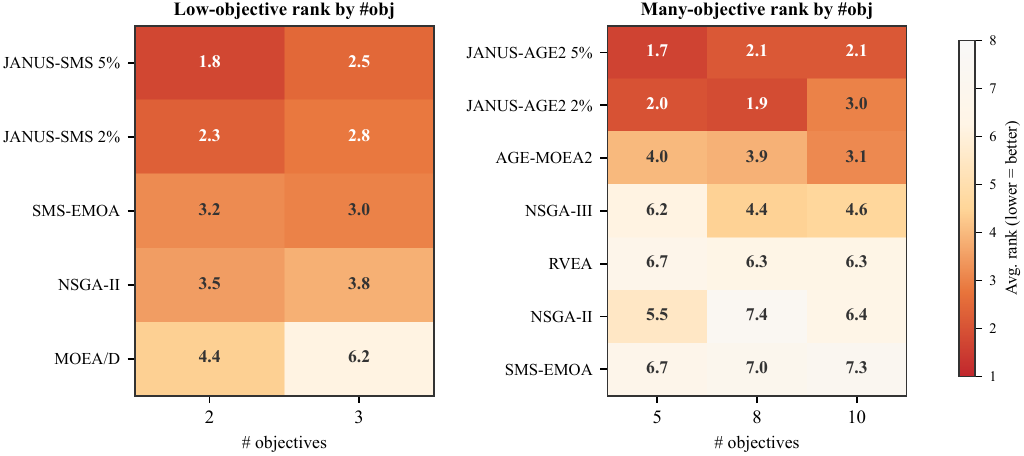}
\caption{\textbf{MOO ranks by objective count.} JANUS-SMS leads low-objective tasks; JANUS-AGE2 leads many-objective tasks.}
\label{fig:moo_matrix}
\end{figure}

\noindent\emph{Readout.} JANUS-SMS leads on 2-objective tasks, where hypervolume contribution is most discriminative and a small number of geometry-guided candidates can quickly improve the dominated region. JANUS-AGE2 leads at 5 and 8 objectives, where many-objective survival has more difficulty preserving useful search directions from ranking alone. Table~\ref{tab:moo_by_nobj_combined} reports the corresponding grouped leaderboard. At 10 objectives the gap narrows because AGE-MOEA2 survival is already strong, leaving less room for proposal-side geometry to change the final ranking.

\begin{table}[H]
\centering
\wideappendixtablesetup
\caption{MOO performance by objective count. Top three methods are shown for each group; JANUS variants lead the low-objective SMS regime and the many-objective AGE2 regime. Blue cell$=$best, green cell$=$second-best per group.}
\label{tab:moo_by_nobj_combined}
\begin{tabular*}{\linewidth}{@{\extracolsep{\fill}} llrrrr}
\toprule
Group & Method & Tasks & Avg. rank & Wins & Norm. HV \\
\midrule
\multicolumn{6}{l}{\emph{Low-objective tasks}} \\
2 & \bestcell{\textbf{JANUS-SMS 5\%}} & 12 & \bestcell{1.75} & \bestcell{8} & \bestcell{1.000} \\
  & \secondcell{\textbf{JANUS-SMS 2\%}} & 12 & \secondcell{2.33} & \secondcell{1} & \secondcell{0.996} \\
  & SMS-EMOA & 12 & 3.17 & 0 & 0.974 \\
3 & \bestcell{\textbf{JANUS-SMS 5\%}} & 6 & \bestcell{2.50} & \bestcell{2} & \bestcell{0.982} \\
  & \secondcell{\textbf{JANUS-SMS 2\%}} & 6 & \secondcell{2.83} & \secondcell{3} & \secondcell{0.986} \\
  & SMS-EMOA & 6 & 3.00 & 0 & 0.967 \\
\midrule
\multicolumn{6}{l}{\emph{Many-objective tasks}} \\
5 & \bestcell{\textbf{JANUS-AGE2 5\%}} & 6 & \bestcell{1.67} & \bestcell{2} & \bestcell{0.997} \\
  & \secondcell{\textbf{JANUS-AGE2 2\%}} & 6 & \secondcell{2.00} & \secondcell{2} & \secondcell{0.983} \\
  & AGE-MOEA2 & 6 & 4.00 & 1 & 0.954 \\
8 & \bestcell{\textbf{JANUS-AGE2 2\%}} & 7 & \bestcell{1.86} & \bestcell{3} & \bestcell{0.975} \\
  & \secondcell{\textbf{JANUS-AGE2 5\%}} & 7 & \secondcell{2.14} & \secondcell{2} & \secondcell{0.959} \\
  & AGE-MOEA2 & 7 & 3.86 & 0 & 0.928 \\
10 & \bestcell{\textbf{JANUS-AGE2 5\%}} & 7 & \bestcell{2.14} & \bestcell{2} & \bestcell{0.965} \\
   & \secondcell{\textbf{JANUS-AGE2 2\%}} & 7 & \secondcell{3.00} & \secondcell{1} & \secondcell{0.963} \\
   & AGE-MOEA2 & 7 & 3.14 & 2 & 0.869 \\
\bottomrule
\end{tabular*}
\end{table}

\paragraph{Ablation readout.} The five-arm ablation is reported in the main paper. Here we keep the broader rank and runtime panels, which show that the zero-regression pattern and modest CUDA overhead hold beyond the compact main-table summary.

\begin{figure}[H]
\centering
\includegraphics[width=0.78\linewidth]{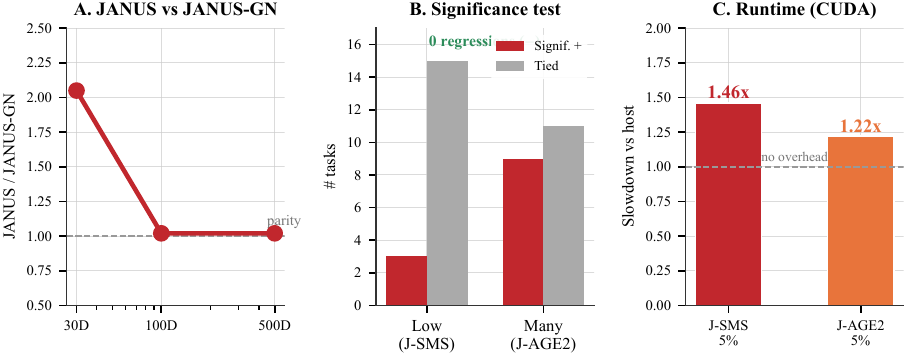}
\caption{\textbf{Ablation, significance, and cost.} (A) JANUS/JANUS-GN ratio. (B) Significance: 0 regressions. (C) CUDA runtime: $1.46\times$/$1.22\times$.}
\label{fig:ablation_panel}
\end{figure}

\paragraph{CMIX fraction and runtime.} Table~\ref{tab:moo_cmix_frac_ablation} sweeps the fraction of offspring allocated to CMIX infill from 0.05 to 0.60. The sweep shows that the best fixed fraction is often larger than the conservative default, but the curve is not monotone on every problem. We therefore use a conservative fraction in the main experiments rather than tuning it per benchmark. Table~\ref{tab:runtime_distribution_pretty} reports wall-clock overhead: with CUDA-accelerated metric estimation, JANUS-SMS 5\% is $1.46\times$ slower than SMS-EMOA and JANUS-AGE2 5\% is $1.22\times$ slower than AGE-MOEA2. This overhead is the price of online Jacobian estimation; it is acceptable when evaluations dominate runtime, but may matter for benchmarks with sub-millisecond evaluation cost.

\begin{table}[H]
\centering
\wideappendixtablesetup
\caption{CMIX candidate fraction ablation (mean HV; higher is better). Blue cell$=$best per row, green cell$=$second-best.}
\label{tab:moo_cmix_frac_ablation}
\begin{tabular*}{\linewidth}{@{\extracolsep{\fill}} lrrrrrrr}
\toprule
Problem & 0.05 & 0.10 & 0.20 & 0.30 & 0.40 & 0.50 & 0.60 \\
\midrule
DTLZ4\_d12\_n3 & 72.7 & 72.9 & 72.9 & \secondcell{72.8} & 72.9 & \bestcell{72.9} & 72.9 \\
WFG1\_d12\_n2 & 8.10 & 9.23 & \secondcell{9.64} & 9.50 & 9.63 & 9.51 & \bestcell{9.86} \\
WFG2\_d12\_n3 & 117 & 119 & 121 & 124 & \bestcell{131} & 127 & \secondcell{128} \\
WFG4\_d12\_n2 & 7.02 & 7.51 & 7.47 & 7.99 & 7.71 & \secondcell{8.07} & \bestcell{8.74} \\
ZDT1\_d30\_n2 & 6.70 & 6.75 & 6.81 & 6.94 & \secondcell{7.09} & 7.07 & \bestcell{7.10} \\
ZDT2\_d30\_n2 & 5.85 & 5.90 & 5.88 & 5.91 & 6.12 & \bestcell{6.29} & \secondcell{6.27} \\
\midrule
\textbf{ALL} & 36.2 & 36.9 & 37.3 & 37.8 & \bestcell{39.1} & 38.4 & \secondcell{38.9} \\
\bottomrule
\end{tabular*}
\end{table}

\begin{table}[H]
\centering
\wideappendixtablesetup
\caption{Runtime overhead with CUDA-accelerated metric estimation. Slowdown is the ratio of mean wall-clock time per independent run for the JANUS variant to its host baseline (lower is better; $1.00\times$ = no overhead). Values match Fig.~\ref{fig:ablation_panel}C.}
\label{tab:runtime_distribution_pretty}
\begin{tabular*}{\linewidth}{@{\extracolsep{\fill}} llrrr}
\toprule
JANUS variant & Host baseline & Tasks & Runs & Slowdown \\
\midrule
\textbf{JANUS-SMS 5\%} & SMS-EMOA & 18 & 180 & \bestcell{$1.46\times$} \\
\textbf{JANUS-AGE2 5\%} & AGE-MOEA2 & 20 & 200 & \secondcell{$1.22\times$} \\
\bottomrule
\end{tabular*}
\end{table}

\paragraph{Per-task winner summary.} Across all 38 MOO tasks, JANUS variants win 16/38 (10/18 low-objective and 6/20 many-objective) by best average rank, not gated by a significance test; the main paper's headline ``12/38 significant wins'' is the subset of these that additionally passes a Mann--Whitney test at $p<0.05$ against the strongest host baseline. The few non-JANUS wins concentrate on simple ZDT instances with disconnected Pareto fronts and on 10-objective tasks where AGE-MOEA2's survival is already highly competitive. This distribution is useful for interpreting the method: JANUS is not a universal dominance trick, but an infill mechanism that helps when the host leaves exploitable local geometry unused.
\noindent\emph{Practical implication.} The MOO appendix shows that JANUS is not uniformly better by construction; its gains appear where the host survival operator leaves room for geometry-guided proposal generation. In practice, this suggests using JANUS as a host augmentation and monitoring whether the host is already saturating diversity preservation before increasing the CMIX fraction.
\FloatBarrier

\subsection{UAV Per-Terrain Results}
\label{app:uav}

\paragraph{Takeaway.} UAV is the clearest structured-objective case because the objective decomposes into meaningful path-quality components and the dimensionality is high enough that blind covariance learning is difficult. JANUS has the best aggregate cost and wins many terrains, but the terrain panel also exposes real failures under the short 2,500-FE budget.

\begin{table}[H]
\centering
\maintablesetup
\caption{UAV budget sweep over 28 terrains and 30 runs. Lower is better.}
\label{tab:uav_long_budget_main}
\begin{tabular*}{\linewidth}{@{\extracolsep{\fill}}lrrr}
\toprule
Budget & \textbf{JANUS} & CMA-ES & Gain \\
\midrule
2,500 & \bestcell{9,600} & 14,486 & $-33.7\%$ \\
10,000 & \bestcell{8,488} & 8,587 & $-1.2\%$ \\
20,000 & \bestcell{8,019} & 8,023 & $-0.05\%$ \\
\bottomrule
\end{tabular*}
\end{table}

\begin{table}[H]
\centering
\wideappendixtablesetup
\caption{UAV path planning: 28 held-out terrains, 30 runs, 2{,}500 FEs (840 values/method). Blue cell$=$best, green cell$=$second-best. Ratio is mean/JANUS (lower=closer to JANUS). \emph{Sig.} column: Bonferroni-corrected Mann--Whitney $U$ wins/ties/losses vs JANUS ($p<0.05/28$ across 28 terrains). The test is unpaired because the CMA-ES baseline runs do not share seeds with the JANUS runs.}
\label{tab:uav_dense}
\resizebox{\linewidth}{!}{%
\begin{tabular*}{\linewidth}{@{\extracolsep{\fill}} clllcccccc}
\toprule
\# & Method & Family & \makecell{Mean\\$\pm$Std} & Median & IQR & \makecell{Ratio\\vs JANUS} & \makecell{Terrains\\rank 1} & \makecell{Sig. vs\\JANUS} \\
\midrule
1 & \textbf{JANUS} & Ours & \bestcell{\makecell{9,600 \\ $\pm$5,966}} & \bestcell{6,644} & 9,579 & \bestcell{1.00} & 15 & -- \\
2 & DE & Traditional & \secondcell{\makecell{11,006 \\ $\pm$3,865}} & \secondcell{9,261} & 7,412 & 1.15 & 1 & 24/3/1 \\
3 & JDE21 & Adaptive & \makecell{11,017 \\ $\pm$5,604} & 8,579 & 2,795 & 1.15 & 1 & -- \\
4 & \textbf{JANUS-GN} & Ablation & \makecell{11,317 \\ $\pm$5,373} & 9,508 & 9,815 & 1.18 & 8 & 0/28/0 \\
5 & RLDEAFL & MetaBBO & \makecell{11,411 \\ $\pm$5,818} & 8,197 & 9,227 & 1.19 & 0 & -- \\
6 & GLEET & MetaBBO & \makecell{11,603 \\ $\pm$7,085} & 7,455 & 9,808 & 1.21 & 1 & -- \\
7 & PSO & Traditional & \makecell{11,792 \\ $\pm$4,454} & 9,156 & 8,223 & 1.23 & 0 & -- \\
8 & RS & Traditional & \makecell{13,891 \\ $\pm$6,859} & 10,394 & 9,169 & 1.45 & 0 & 23/5/0 \\
9 & CMA-ES & Adaptive & \makecell{14,486 \\ $\pm$6,559} & 10,908 & 9,956 & 1.51 & 0 & 24/3/1 \\
10 & SAHLPSO & Adaptive & \makecell{14,865 \\ $\pm$6,885} & 11,265 & 9,716 & 1.55 & 0 & -- \\
11 & GLHF & MetaBBO & \makecell{20,354 \\ $\pm$5,918} & 20,921 & 7,137 & 2.12 & 0 & -- \\
12 & LES & MetaBBO & \makecell{26,968 \\ $\pm$14,006} & 26,220 & 20,365 & 2.81 & 2 & -- \\
\bottomrule
\end{tabular*}
}
\end{table}

\noindent\emph{Readout.} JANUS is most useful in the early-budget regime. At 2,500 evaluations, the local metric gives the host a faster route toward feasible, low-cost path shapes, as also reflected in the final-value distribution in Fig.~\ref{fig:uav_dist}. At longer budgets, CMA-ES catches up and the advantage converges toward parity, so JANUS should be interpreted as an early-budget accelerator rather than an asymptotic replacement. This is the regime we care about most for expensive simulation-driven planning.

\noindent
\begin{minipage}{\linewidth}
\centering
\includegraphics[width=0.58\linewidth]{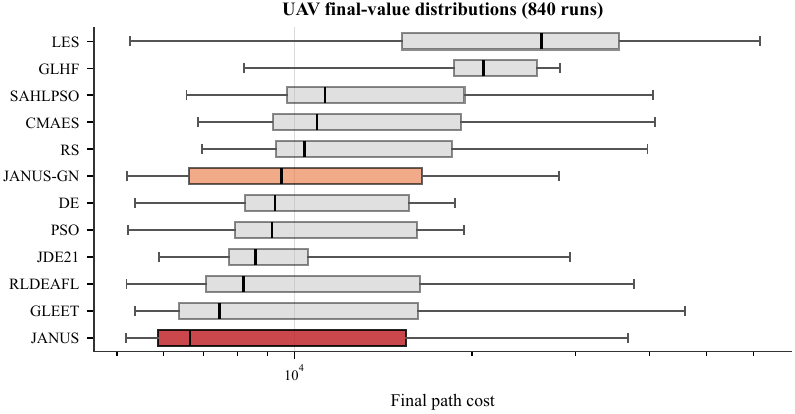}
\captionof{figure}{UAV final-value distributions over 840 runs.}
\label{fig:uav_dist}
\end{minipage}

\noindent\emph{Readout.} JANUS has the best center of mass, while several baselines occasionally find strong terrain-specific paths. The distribution plot is therefore not meant to claim pointwise dominance; it shows that JANUS shifts the bulk of outcomes downward while still leaving a nontrivial tail of terrain-specific failures. This motivates the terrain-level audit rather than relying only on aggregate averages.
\paragraph{3D path visualizations.} The main paper now includes the representative T1/T17 convergence and path panels. Figure~\ref{fig:uav_paths_app} therefore keeps only additional terrain-level path examples: T9/T41 are hard cases where non-JANUS baselines can find safer early paths, while T25/T55 are representative JANUS wins. This avoids duplicating the main figure while keeping the visual evidence aligned with the terrain-level audit.

\begin{figure}[H]
\centering
\begin{subfigure}{0.24\linewidth}\centering
\includegraphics[width=\linewidth,trim=20pt 32pt 14pt 0pt,clip]{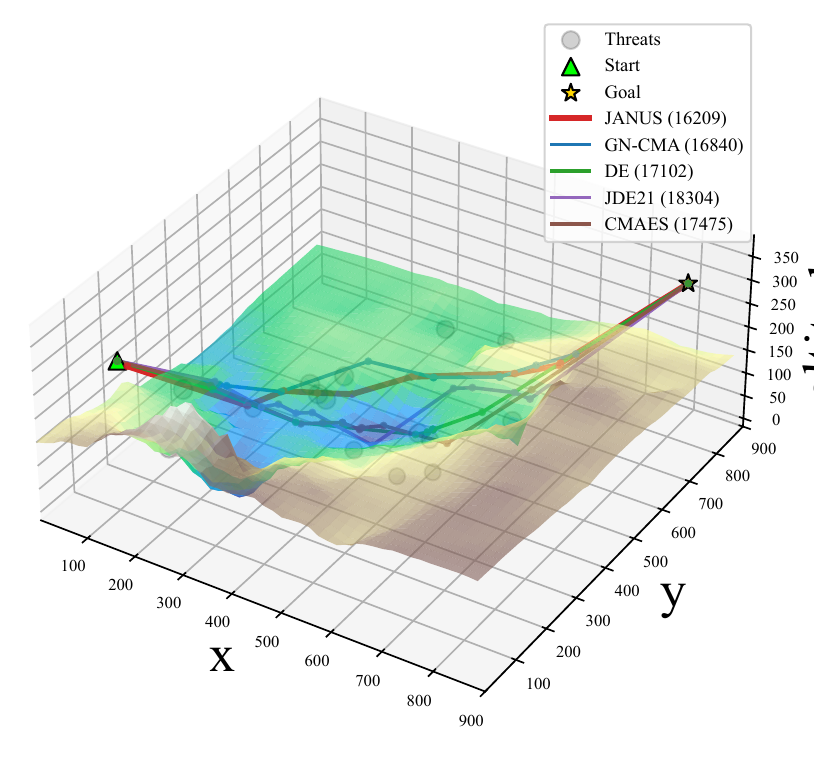}
\caption{T9: hard}
\end{subfigure}\hfill
\begin{subfigure}{0.24\linewidth}\centering
\includegraphics[width=\linewidth,trim=20pt 32pt 14pt 0pt,clip]{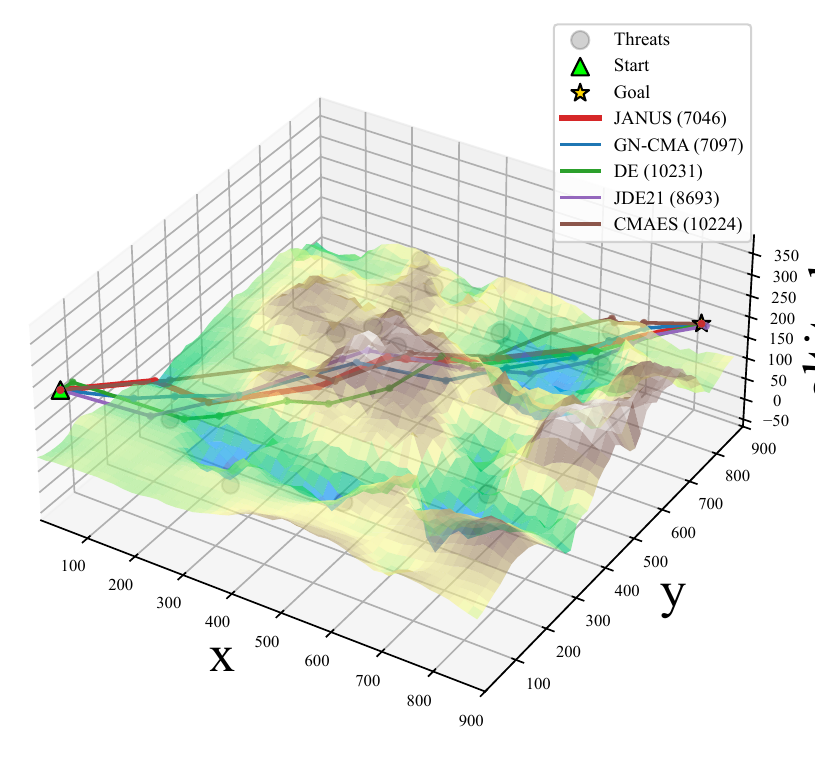}
\caption{T25: JANUS win}
\end{subfigure}\hfill
\begin{subfigure}{0.24\linewidth}\centering
\includegraphics[width=\linewidth,trim=20pt 32pt 14pt 0pt,clip]{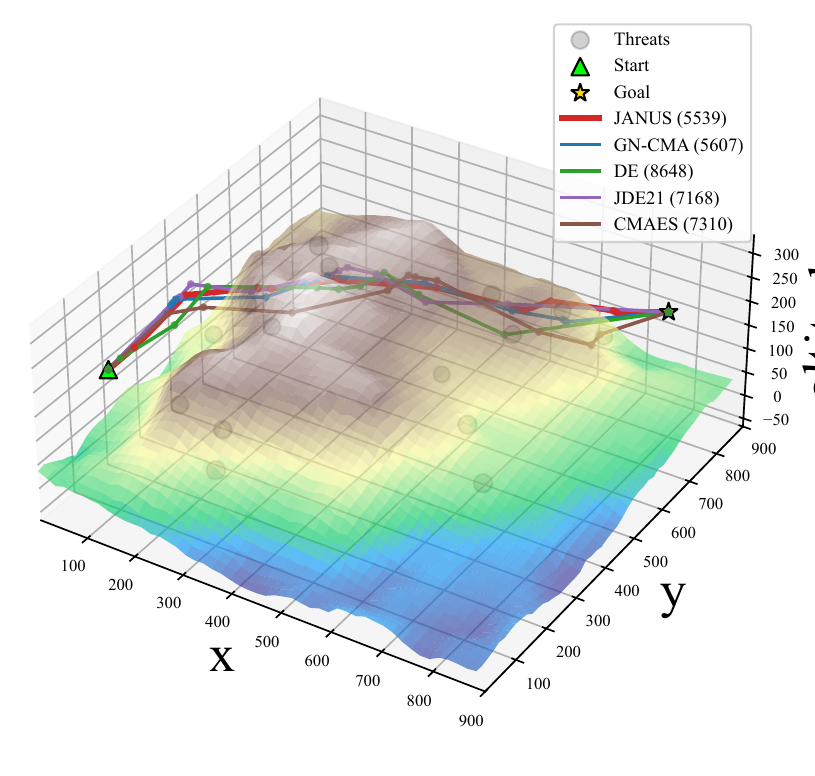}
\caption{T41: hard}
\end{subfigure}\hfill
\begin{subfigure}{0.24\linewidth}\centering
\includegraphics[width=\linewidth,trim=20pt 32pt 14pt 0pt,clip]{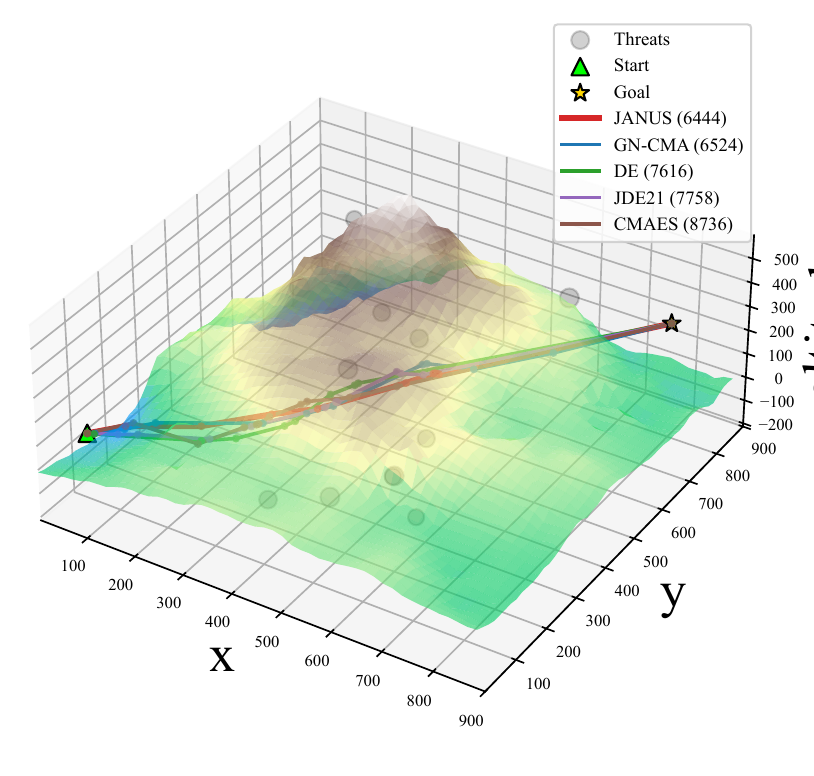}
\caption{T55: JANUS win}
\end{subfigure}
\caption{\textbf{Additional UAV 3D path visualizations} on four terrains (T9/T25/T41/T55). JANUS often finds smooth, low-cost trajectories, but T9/T41 show that this behavior is not universal under the short 2,500-FE budget.}
\label{fig:uav_paths_app}
\end{figure}
\noindent\emph{Terrain audit.} Table~\ref{tab:dense_uav_terrain_panel} is intentionally a compact audit rather than a visual centerpiece: it keeps every held-out terrain visible and makes the significant wins, ties, and the one significant loss against the CMA-ES host easy to locate. We keep the non-wins in the same table as the wins because they explain when CMIX should be used cautiously: the one true failure occurs where CMIX's covariance reshaping actively undoes a strong exploitation signal.
\begin{table}[H]
\centering
\wideappendixtablesetup
\caption{UAV terrain audit. Every held-out terrain is kept; gain compares JANUS with the CMA-ES host under the identical 2,500-FE protocol. Cell color reflects statistical significance (Bonferroni-corrected Mann--Whitney $U$ across 28 terrains, consistent with Table~\ref{tab:uav_dense}'s Sig.\ column): green$=$significant win, amber$=$not significant (tie), red$=$significant loss.}
\label{tab:dense_uav_terrain_panel}
\begin{tabular*}{0.49\linewidth}{@{\extracolsep{\fill}} lrrr}
\toprule
Terrain & JANUS & CMA-ES & Gain \\
\midrule
T1 & 7035 & 10062 & \cellcolor{janusgreen!55}+30.1\% \\
T3 & 9716 & 17753 & \cellcolor{janusgreen!55}+45.3\% \\
T5 & 5974 & 19695 & \cellcolor{janusgreen!55}+69.7\% \\
T7 & 5834 & 18681 & \cellcolor{janusgreen!55}+68.8\% \\
T9 & 23155 & 19402 & \cellcolor{janusamber!30}-19.3\% \\
T11 & 21077 & 18970 & \cellcolor{janusamber!30}-11.1\% \\
T13 & 7527 & 10261 & \cellcolor{janusgreen!55}+26.6\% \\
T15 & 6528 & 8369 & \cellcolor{janusgreen!55}+22.0\% \\
T17 & 6584 & 20170 & \cellcolor{janusgreen!55}+67.4\% \\
T19 & 22186 & 19407 & \cellcolor{janusamber!30}-14.3\% \\
T21 & 7515 & 10439 & \cellcolor{janusgreen!55}+28.0\% \\
T23 & 5807 & 8322 & \cellcolor{janusgreen!55}+30.2\% \\
T25 & 11623 & 18375 & \cellcolor{janusgreen!55}+36.7\% \\
T27 & 15966 & 27321 & \cellcolor{janusgreen!55}+41.6\% \\
\bottomrule
\end{tabular*}\hfill
\begin{tabular*}{0.49\linewidth}{@{\extracolsep{\fill}} lrrr}
\toprule
Terrain & JANUS & CMA-ES & Gain \\
\midrule
T29 & 15998 & 10607 & \cellcolor{red!12}-50.8\% \\
T31 & 6044 & 18533 & \cellcolor{janusgreen!55}+67.4\% \\
T33 & 7696 & 10066 & \cellcolor{janusgreen!55}+23.5\% \\
T35 & 5304 & 7547 & \cellcolor{janusgreen!55}+29.7\% \\
T37 & 10520 & 18362 & \cellcolor{janusgreen!55}+42.7\% \\
T39 & 10395 & 17266 & \cellcolor{janusgreen!55}+39.8\% \\
T41 & 6621 & 9257 & \cellcolor{janusgreen!55}+28.5\% \\
T43 & 6119 & 16738 & \cellcolor{janusgreen!55}+63.4\% \\
T45 & 5975 & 8912 & \cellcolor{janusgreen!55}+32.9\% \\
T47 & 8285 & 9291 & \cellcolor{janusgreen!55}+10.8\% \\
T49 & 5690 & 8906 & \cellcolor{janusgreen!55}+36.1\% \\
T51 & 9329 & 11420 & \cellcolor{janusgreen!55}+18.3\% \\
T53 & 7572 & 10028 & \cellcolor{janusgreen!55}+24.5\% \\
T55 & 6728 & 21444 & \cellcolor{janusgreen!55}+68.6\% \\
\bottomrule
\end{tabular*}
\end{table}

\noindent\emph{Readout.} Under this host-controlled comparison (optimization budget, terrain set, and number of runs matched across methods as in the BBOB and MOO tables; the statistical test is the Bonferroni-corrected \emph{unpaired} Mann--Whitney $U$ across 28 terrains, chosen because these runs are not seed-paired), JANUS significantly beats the CMA-ES host on 24/28 terrains, ties on 3 (T9, T11, T19; the raw mean favors CMA-ES but the difference is not significant), and loses significantly on exactly 1 (T29). This matches Table~\ref{tab:uav_dense}'s aggregate Sig.\ column for CMA-ES (24/3/1) exactly. On all 4 of these non-win terrains, JANUS-GN (exploitation only) ranks first among all 12 compared methods and clearly beats CMA-ES, while the full JANUS variant falls to 5th--9th place; the GN signal itself is strong there, but CMIX's covariance reshaping is harmful, most severely on T29 (JANUS-GN 6{,}063 vs.\ full JANUS 15{,}998, CMA-ES 10{,}607, $p<10^{-5}$). This is consistent with JANUS-GN ranking first on 8 terrains overall in the main paper's ``ranks first on 15/28'' count (computed against all 12 methods): these 4 terrains are among the 8, confirming that CMIX should be used cautiously when the recent trace has not yet produced a stable curvature estimate. This distinction matters operationally: GN injection is a safer fallback when the path distribution is still changing rapidly, while CMIX is most useful after the recent window contains coherent path variations.
\FloatBarrier

\section{Negative Results and Operating Boundaries}
\label{app:negative}

\paragraph{Operating regime.} Sec.~\ref{sec:implementation} summarizes the operating regime found across experiments; we do not repeat it here. The mechanism-level detail worth adding is that, in the high-dimensional BBOB and UAV regimes where JANUS helps most, CMIX is not simply adding more candidates: it reallocates the same sampling scale toward directions that the fitted residual model regards as locally flat or under-explored.

\paragraph{Takeaway.} JANUS is a local-geometry method. The inverse Hessian metric $T=(G+\lambda I)^{-1}$ captures the landscape near the current mean, but cannot reveal disconnected basins outside the reachable radius or outside the trust region of the local Jacobian estimate. The failure cases summarized in Table~\ref{tab:negative_evidence} are therefore expected rather than incidental. JANUS is strongest when the recent trace gives a reliable metric, while JANUS-GN is safer when basin discovery, unstable curvature, or noisy pseudo-residuals dominate. These negative cases are included to make the method's operating boundary explicit.

\begin{table}[H]
\centering
\wideappendixtablesetup
\caption{Negative evidence summary: cases where JANUS underperforms or ties.}
\label{tab:negative_evidence}
\begin{tabular*}{\textwidth}{@{\extracolsep{\fill}} l p{0.28\linewidth} p{0.46\linewidth}}
\toprule
Regime & Evidence & Interpretation \\
\midrule
BBOB 30D/100D & CMIX loses more often than JANUS-GN & Scalar pseudo-residual curvature can be noisy in low dimension. \\
BBOB 500D & Host-control losses on $f_{22}$/$f_{23}$ & Local metric helps broadly but not when basin discovery dominates. \\
ABOM reference & $f_{19}$/$f_{20}$ anomalous ratios & Reported-reference comparisons are interpreted qualitatively. \\
UAV short budget & 1 significant host-control loss (T29), 3 ties (T9/11/19); JANUS-GN ranks first on all 4 & CMIX can undo a strong GN signal before curvature stabilizes. \\
MOO saturation & 4 tied tasks & Survival already preserves diversity, limiting proposal-geometry gains. \\
\bottomrule
\end{tabular*}
\end{table}

\paragraph{How to read the failures.} The BBOB losses do not show that the GN candidate is harmful in general. They show that scalar pseudo-residual curvature can be misleading when the useful move is global, not local. The host-controlled comparison points mainly to $f_{22}$/$f_{23}$; the ABOM reported-reference ratios add a few anomalous cases that we treat qualitatively instead of as same-protocol failures.

The UAV losses have a different cause. The residual structure is meaningful, but the 2,500-FE budget can be too short for a stable curvature estimate. The MOO ties occur when the host survival operator already preserves enough diversity, leaving little room for proposal-side geometry. These failure modes call for different responses: restarts for disconnected basins, GN-only fallback for unstable curvature, and conservative CMIX fractions when the host is already strong.

\paragraph{Practical use pattern.} A practical use pattern follows from this boundary. For scalar objectives, use JANUS-GN or a small CMIX fraction until the recent window produces stable curvature directions. Increase CMIX only when the host is spending many evaluations in a locally coherent region. For structured objectives or multi-objective tasks, where residual coordinates are observed rather than fitted, CMIX can be enabled earlier. In all cases, the host should keep ownership of restarts, boundary handling, and diversity preservation. JANUS is intentionally not a trust-region replacement and not a learned optimizer: it is a small, local, bounded proposal mechanism whose value is highest when evaluations are expensive enough that reusing local geometry matters.

\subsection{Industrial Validation: ISP Calibration ($d{=}1135$)}
\label{app:awb}

To complement the synthetic benchmarks, we report a single deployment on a proprietary image signal processing (ISP) calibration task. The objective is the mean reproduction angle error (MRAE) of an auto white balance (AWB) module across $m{=}111$ test scenes; the search space is $d{=}1135$ ISP parameters. Each evaluation runs the full ISP pipeline on all scenes, making repeated large-scale experimentation infeasible. We therefore report a single run per method under a fixed 5,000-evaluation budget.

\begin{table}[H]
\centering
\appendixtablesetup
\caption{ISP calibration (AWB, $d{=}1135$, $m{=}111$ scenes): running-best MRAE. Lower is better. Single run per method, 5,000 evaluations. Blue cell$=$best, green cell$=$second-best per column.}
\label{tab:awb}
\begin{tabular*}{\textwidth}{@{\extracolsep{\fill}} lrrrrr}
\toprule
Method & @1000 & @2000 & @3000 & @4000 & @5000 \\
\midrule
CMA-ES (baseline) & 1.529 & 1.276 & 1.177 & 1.127 & 1.056 \\
JANUS-GN (exploit only) & \bestcell{1.154} & \bestcell{0.968} & \secondcell{0.918} & \secondcell{0.866} & \secondcell{0.837} \\
JANUS (GN + CMIX) & \secondcell{1.200} & \secondcell{0.993} & \bestcell{0.897} & \bestcell{0.863} & \bestcell{0.819} \\
\bottomrule
\end{tabular*}
\end{table}

\noindent\emph{Readout.} Both JANUS variants reduce MRAE by ${\sim}21$--$22\%$ relative to CMA-ES at @5000. The full JANUS variant (GN+CMIX) trails the GN-only variant early (@1000) because it allocates sampling energy to exploration, but overtakes it from @3000 onward and achieves the lowest final error (0.819 vs.\ 0.837). The CMIX advantage continues to grow at @5000 with no sign of reversal, consistent with the delayed benefit of CMIX observed in the controlled ablation (Sec.~\ref{sec:exp-rq4}). Because this is a single run on a proprietary pipeline, we treat it as deployment evidence rather than a statistically controlled comparison.
\FloatBarrier

\subsection{Derivation Notes}
\label{app:proofs}

\paragraph{Residual coordinates to GN.} For component observations, JANUS starts from the form $f(x)=m^{-1}\sum_i r_i(x)$ and defines $s_i(x)=\sqrt{r_i(x)}$, so $f(x)=m^{-1}\|s(x)\|^2$. This is the standard least-squares view behind Gauss--Newton. Scalar-only objectives do not expose true components, but JANUS can still construct a local residual coordinate system by fitting linear or quadratic models to the recent trace. The resulting residual map is only local, but that is sufficient for an infill module whose candidates are clipped to the host sampling region. With $J=\partial s/\partial x$, the local first- and second-order approximations are
\[
\nabla f(x)\approx \frac{2}{m}J^\top s(x),\qquad H_{\rm GN}(x)\approx \frac{2}{m}J^\top J.
\]
Ignoring the common factor $2/m$ and adding Levenberg--Marquardt damping gives the injected exploitation candidate
\[
(J^\top J+\lambda I)\Delta x=-J^\top s.
\]
Thus explicit residual components are useful but not required: scalar JANUS uses the same algebra with a fitted local residual chart.

\paragraph{Black-box Jacobian estimation.} JANUS estimates $J$ only from the recent search trace. Let $X_W$ be the current window, let $U_r$ be the local PCA basis, and write each centered sample as $z=U_r^\top(x-\bar x)$. For residual observations $S$ and standardized design matrix $Z$, ridge regression gives
\[
\hat B=(Z^\top Z+\rho I)^{-1}Z^\top S,
\]
where columns of $\hat B$ are residual slopes in PCA coordinates. Mapping back through $U_r$ gives the ambient Jacobian estimate used by both GN and CMIX. The PCA restriction is not a modeling assumption about the true objective; it is a variance-control device that prevents underdetermined high-dimensional regressions from dominating the host optimizer.

\paragraph{CMIX metric and dimension scaling.} CMIX forms the inverse Hessian metric $T=(J^\top J+\lambda I)^{-1}$, normalizes its trace to match the host covariance scale, and mixes it with the host covariance only through a bounded fraction. A fixed mixing weight is too small in very high dimension and too aggressive in low dimension, so JANUS uses
\[
w=w_{\rm base}\frac{\min(m,|\mathcal W|)}{d},\qquad w\in[w_{\min},w_{\max}].
\]
For UAV, $m=111$, $|\mathcal W|=300$, and $d=1135$, yielding $w\approx0.049$. This explains why the UAV experiments require a dimension-aware CMIX weight: a fixed $3\times10^{-6}$-style fraction would make the metric contribution effectively invisible.

\noindent\textbf{Proof of Proposition~\ref{prop:restricted}.} Write the SVD $J=U\Sigma V^\top$ with $\Sigma=\operatorname{diag}(\sigma_1,\ldots,\sigma_r)$, $r=\operatorname{rank}(J)$. The inverse Hessian metric is $T=(J^\top J+\lambda I)^{-1}=V(\Sigma^\top\Sigma+\lambda I)^{-1}V^\top$, so each right-singular direction $v_i$ is an eigenvector of $T$ with eigenvalue $(\sigma_i^2+\lambda)^{-1}$, and directions in $\ker(J)$ are eigenvectors with eigenvalue $\lambda^{-1}$. Since every eigenvalue satisfies $(\sigma_i^2+\lambda)^{-1}\ge(\sigma_{\max}(J)^2+\lambda)^{-1}$ and $\lambda^{-1}\ge(\sigma_{\max}(J)^2+\lambda)^{-1}$, the metric is uniformly positive definite:
\begin{equation}
T\ \succeq\ \frac{1}{\sigma_{\max}(J)^2+\lambda}\,I.
\label{eq:T-pd}
\end{equation}

\emph{The restricted covariance annihilates $\operatorname{span}(B)^\perp$ from both sides.} Any restricted-covariance update
$C'=(1-\alpha)C_{\rm host}+\alpha B M' B^\top$ with $C_{\rm host}=BMB^\top$ can be written $C'=B\big[(1-\alpha)M+\alpha M'\big]B^\top=:B\widetilde M B^\top$, whose range and corange lie in $\operatorname{span}(B)$. With $B$ having orthonormal columns, $P_\perp=I-BB^\top$ is the orthogonal projector onto $\operatorname{span}(B)^\perp$ and $B^\top P_\perp=B^\top-(B^\top B)B^\top=0$. Hence
\begin{equation}
P_\perp C' P_\perp \;=\; (P_\perp B)\widetilde M (B^\top P_\perp) \;=\; 0
\label{eq:cprime-zero}
\end{equation}
regardless of $M$, $M'$, and $\alpha$: the restricted family has no degrees of freedom in the complement of its own basis. Note that this argument uses only that $C'$ is spanned by $B$; it does \emph{not} require any right-singular vector of $J$ to lie exactly in $\operatorname{span}(B)^\perp$.

\emph{A projection-based Frobenius bound.} Since $P_\perp$ is an orthogonal projector, the map $A\mapsto P_\perp A P_\perp$ is a Frobenius-norm contraction ($\|P_\perp A P_\perp\|_F\le\|A\|_F$). Applying it to $A=C'-T$ and using \eqref{eq:cprime-zero},
\begin{equation}
\|C'-T\|_F\ \ge\ \|P_\perp(C'-T)P_\perp\|_F\ =\ \|P_\perp T P_\perp\|_F.
\label{eq:proj-bound}
\end{equation}
It remains to lower-bound $\|P_\perp T P_\perp\|_F$. Because $T$ is symmetric, congruence by $P_\perp$ preserves the Loewner order, so \eqref{eq:T-pd} gives $P_\perp T P_\perp\succeq(\sigma_{\max}(J)^2+\lambda)^{-1}P_\perp$. For positive semidefinite matrices $X\succeq Y\succeq0$ we have $\|X\|_F\ge\|Y\|_F$ (as $\|X\|_F^2-\|Y\|_F^2=\langle X-Y,X+Y\rangle\ge0$), and $\|P_\perp\|_F=\sqrt{\operatorname{rank}(P_\perp)}=\sqrt{d-k}$. Therefore
\begin{equation}
\|C'-T\|_F\ \ge\ \|P_\perp T P_\perp\|_F\ \ge\ \frac{\|P_\perp\|_F}{\sigma_{\max}(J)^2+\lambda}\ =\ \frac{\sqrt{d-k}}{\sigma_{\max}(J)^2+\lambda}\ >\ 0,
\label{eq:final-bound}
\end{equation}
strictly positive because $k<d$. In particular $C'\ne T$. The bound has the correct units of an inverse curvature ($1/\sigma^2$), scales with the dimension deficit $d-k$ of the restricted basis, and does not depend on $M,M',\alpha$. The middle quantity $\|P_\perp T P_\perp\|_F$ is the exact residual-aligned curvature of $T$ restricted to $\operatorname{span}(B)^\perp$, which the host sets identically to zero. $\square$

\emph{Remark.} The proposition does not claim that LM-CMA or VkD-CMA underperform JANUS in absolute terms; their host adaptation may be sufficient when the trace itself concentrates in $\operatorname{span}(B)$. The claim is the structural one stated in the introduction: restricted-covariance hosts cannot, in general, express the residual-aligned curvature that JANUS derives from $J$. The empirical question of when this gap matters is taken up in Sec.~\ref{sec:experiments}.

\emph{Scope: purely rank-$k$ hosts.} Prop.~\ref{prop:restricted} assumes a \emph{purely} rank-$k$ covariance $C_{\rm host}=BMB^\top$ with $k<d$: the annihilation identity \eqref{eq:cprime-zero} relies on $C'$ having range and corange inside $\operatorname{span}(B)$. If a restricted-covariance implementation adds an isotropic floor $\epsilon I$ or a full-rank diagonal term (as some practical LM-CMA/VkD-CMA variants do for numerical stability), then $C'$ no longer annihilates $\operatorname{span}(B)^\perp$ and the exact-mismatch statement weakens to an $O(\epsilon)$-approximate one; the proposition should be read as governing the rank-$k$ low-rank correction, not every implementation detail. The hypothesis also fails for a full-rank host such as vanilla CMA-ES, where $\operatorname{span}(B)=\mathbb{R}^d$, $P_\perp=0$, and the bound is vacuously zero; there the proposition makes no structural claim. This distinction separates two questions that should not be conflated. Table~\ref{tab:restricted_cov_baselines} addresses the \emph{structural} one---whether a restricted-covariance host can match a full-rank host at all---and answers it in the negative, consistent with Prop.~\ref{prop:restricted}. Whether JANUS additionally improves on a full-rank host is a separate, \emph{empirical} convergence-speed question, and is not structural: learning a general full covariance from rank-based updates involves $O(d^2)$ covariance parameters and can therefore become sample- and computation-intensive in high dimensions~\citep{hansen2016cma}, whereas JANUS computes $T=(J^\top J+\lambda I)^{-1}$ from the fitted $J$ in one step and mixes it in immediately. That convergence-speed comparison is quantified separately, on real 30-seed data, in the DFO panel of Table~\ref{tab:new_baselines}.

\paragraph{Propositions restated with full proofs.}
Propositions~\ref{prop:descent}--\ref{prop:trace} restate Propositions~\ref{prop:descent-main} and \ref{prop:safety-main} from the main paper (the latter split into two parts) so that complete proofs can be given below. Proposition~\ref{prop:blockwise} is the only new result in this section.
\begin{proposition}[Blockwise assembly is a degenerate GN case]
\label{prop:blockwise}
Partition coordinates into groups and constrain the fitted normal matrix $J^\top J$ to be block diagonal under this partition. Then the damped GN solve $(J^\top J+\lambda I)\Delta=-J^\top s$ decomposes into independent block solves. Per-component best-fragment assembly is recovered as a zero-coupling, zero-damping, discrete-search limit of this blockwise update.
\end{proposition}
\begin{proposition}[Damped GN descends on the fitted residual model]
\label{prop:descent}
For $\hat s(x_b+\Delta)=s_b+J\Delta$ and $\hat f(\Delta)=\|s_b+J\Delta\|_2^2$, the JANUS step $\Delta_{\rm GN}=-(J^\top J+\lambda I)^{-1}J^\top s_b$ satisfies $\nabla\hat f(0)^\top\Delta_{\rm GN}\le0$ for $\lambda>0$, strictly whenever $J^\top s_b\ne0$.
\end{proposition}
\begin{proposition}[Bounded injection controls normalized step length]
\label{prop:bounded}
If an injected point is clipped in host-whitened coordinates to $\|y_e\|\le\kappa\sqrt d$, then $\|x_e-\bar x\|\le\kappa\sigma\lambda_{\max}(C^{1/2})\sqrt d$.
\end{proposition}
\begin{proposition}[Trace-preserving covariance mix]
\label{prop:trace}
For $T\succ0$ and $C'=(1{-}w)C+w\operatorname{tr}(C)T/\operatorname{tr}(T)$, $\operatorname{tr}(C')=\operatorname{tr}(C)$.
\end{proposition}

\noindent\textbf{Proof of Proposition~\ref{prop:blockwise}.} Let the coordinate partition be $\mathcal G=\{G_1,\ldots,G_q\}$ and suppose the fitted normal matrix has no cross-block terms: $J^\top J=\operatorname{blkdiag}(G_1^\top G_1,\ldots,G_q^\top G_q)$. The damped system $(J^\top J+\lambda I)\Delta=-J^\top s$ is then block diagonal, so each $\Delta_{G_j}$ depends only on the residual slopes and residual values assigned to block $G_j$. If damping is taken to zero and the continuous block step is restricted to the finite set of historical fragments available for that block, the update reduces to choosing the fragment that minimizes the corresponding component residual. This is the per-component assembly rule. The condition is deliberately restrictive: once off-diagonal blocks are kept, JANUS can express interactions among coordinate groups that assembly discards.

\noindent\textbf{Proof of Proposition~\ref{prop:descent}.} For the fitted model $\hat f(\Delta)=\|s_b+J\Delta\|_2^2$, $\nabla \hat f(0)=2J^\top s_b$. Substituting $\Delta_{\rm GN}=-(J^\top J+\lambda I)^{-1}J^\top s_b$ gives
\[
\nabla \hat f(0)^\top\Delta_{\rm GN}
=-2(J^\top s_b)^\top(J^\top J+\lambda I)^{-1}(J^\top s_b).
\]
Because $J^\top J+\lambda I\succ0$ for $\lambda>0$, the quadratic form is nonnegative, and it is positive unless $J^\top s_b=0$. Therefore the GN proposal is a descent direction for the fitted residual objective whenever the fitted local gradient is nonzero. This is not a global improvement theorem: if the fitted residual chart is inaccurate outside the window, the true objective can still fail to improve.

\noindent\textbf{Proof of Proposition~\ref{prop:bounded}.} Host-whitened coordinates are defined by $y_e=C^{-1/2}(x_e-\bar x)/\sigma$, equivalently $x_e-\bar x=\sigma C^{1/2}y_e$. Hence
\[
\|x_e-\bar x\|\le \sigma\|C^{1/2}\|_2\|y_e\|
\le \kappa\sigma\lambda_{\max}(C^{1/2})\sqrt d.
\]
Thus clipping controls the injected step in exactly the coordinate system used by the host distribution. The proposition does not claim that JANUS leaves the stochastic step-size path unchanged; it claims the weaker and needed safety property that injected candidates are scale-compatible with ordinary samples.

\noindent\textbf{Proof of Proposition~\ref{prop:trace}.} Let $T'=\operatorname{tr}(C)T/\operatorname{tr}(T)$. Then $\operatorname{tr}(T')=\operatorname{tr}(C)$ by linearity of trace, and therefore
\[
\operatorname{tr}((1-w)C+wT')=(1-w)\operatorname{tr}(C)+w\operatorname{tr}(T')=\operatorname{tr}(C).
\]
CMIX can rotate or anisotropically reshape the proposal covariance through $T'$, but it cannot increase the total covariance trace. These results are local: $T$ can reallocate variance among plausible nearby directions, but disconnected basin discovery still requires host diversity, restarts, or another global exploration mechanism.

\paragraph{Implementation notes.} JANUS is implemented as a bounded candidate generator, not as a replacement optimizer. The host proposes its ordinary offspring; JANUS reads the recent evaluated window, fits the residual chart, and contributes a small number of GN/CMIX candidates. All candidates are evaluated before the host performs its standard tell/update step. Selection, survival, covariance adaptation, step-size adaptation, archive maintenance, and restart logic remain owned by the host.

In CMA-style hosts, a proposed displacement is transformed into host-whitened coordinates and clipped to the same $\sqrt d$ scale as ordinary samples. In DE and MOO hosts, we apply analogous box and step-length constraints before evaluation. This keeps JANUS proposals on the host sampling scale and prevents one ill-conditioned Jacobian estimate from dominating adaptation.

For scalar objectives, the residual chart is rebuilt frequently from local scalar fits with standardized design matrices, ridge regularization, and rank caps for the quadratic chart. For UAV, component costs define a more meaningful residual space; for MOO, the objective vector itself provides the residual coordinates. Changing the host requires only the ask/tell interface, the available residual coordinates, and the fixed hyperparameters in Table~\ref{tab:hparams}. JANUS-GN is the conservative fallback because it uses the same local Jacobian for one bounded exploitation candidate without reshaping the host covariance.

\end{document}